\documentclass[mathpazo]{csam}

\usepackage{amssymb}
\usepackage{latexsym}

\usepackage{url}
\usepackage{xcolor}
\definecolor{newcolor}{rgb}{.8,.349,.1}

\usepackage{algorithm}
\usepackage{algpseudocode}
\usepackage{amssymb,amsmath,amsthm}
\usepackage{graphics}
\usepackage{epsfig}
\usepackage{mathtools}
\usepackage{wrapfig}

\usepackage{tikz}
\usetikzlibrary{positioning}
\usepackage{xcolor}

\newtheorem{thm}{Theorem}[section]

\newtheorem{prop}[thm]{Proposition}

\newtheorem{rem}{Remark}
\newtheorem{ass}{Assumption}

\newcommand{\LS}{\mathrm{LS}}
\newcommand{\dist}{\mathrm{dist}}
\newcommand{\vzero}{\mathbf{0}}

\begin{document}
\title{Unsupervised Deep Learning Meets Chan-Vese Model}


\author[D. Zheng et~al.]{Dihan Zheng\affil{1},
   Chenglong Bao\affil{1,2}\comma\corrauth, Zuoqiang Shi\affil{3,2}, Haibin Ling\affil{4}~and Kaisheng Ma\affil{5}}
\address{
	\affilnum{1}\ Yau Mathematical Sciences Center, Tsinghua University, China. \\
   	\affilnum{2}\ Yanqi Lake Beijing Institute of Mathematical Sciences and Applications, China. \\
   	\affilnum{3}\ Department of Mathematical Sciences, Tsinghua University, China. \\
   	\affilnum{4}\ Department of Computer Sciences, Stony Brook University, USA. \\
   	\affilnum{5}\ Institute for Interdisciplinary Information Sciences, Tsinghua University, China. \\
   }
\emails{{\tt clbao@mail.tsinghua.edu.cn} (C.~Bao)}

\begin{abstract}
The Chan-Vese (CV) model is a classic region-based method in image segmentation. However, its piecewise constant assumption does not always hold for practical applications. Many improvements have been proposed but the issue is still far from well solved. In this work, we propose an unsupervised image segmentation approach that integrates the CV model with deep neural networks, which significantly improves the original CV model's segmentation accuracy. Our basic idea is to apply a deep neural network that maps the image into a latent space to alleviate the violation of the piecewise constant assumption in image space. We formulate this idea under the classic Bayesian framework by approximating the likelihood with an evidence lower bound (ELBO) term while keeping the prior term in the CV model. Thus, our model only needs the input image itself and does not require pre-training from external datasets. Moreover, we extend the idea to multi-phase case and dataset based unsupervised image segmentation. Extensive experiments validate the effectiveness of our model and show that the proposed method is noticeably better than other unsupervised segmentation approaches.
\end{abstract}

\maketitle

\section{Introduction}
Image segmentation is one of the fundamental problems in image processing and has many applications in computer vision such as object detection~\cite{delmerico2011building}, recognition~\cite{tu2005image} and medical image analysis~\cite{pham2000current}. Despite great improvements in image segmentation in recent years, it remains challenging and deserves further exploration. Specifically, given an open and bounded domain $\Omega\subset\mathbb{R}^2$ and an image $I:\Omega\rightarrow \mathbb{R}$, segmentation aims at finding a decomposition of the region $\Omega=(\cup_{i=1}^N\Omega_{i})\cup \Gamma$, where $\Gamma$ is the closed segmentation curve and $\Omega_i, i=1,2,\ldots,N$ are disjoint open regions of interests. One seminar work is the so-called Mumford-Shah model~\cite{mumford1989optimal} that minimizes the following functional:
\begin{equation}\label{Energy:MS}
	E_{\mathrm{MS}}(J, \Gamma) := \int_{\Omega}(I-J)^{2}dx
	+\mu\int_{\Omega \setminus \Gamma}|\nabla J|^{2}dx +\nu |\Gamma|,
\end{equation}
where $J$ is a piecewise smooth approximation of $I$ and $|\Gamma|$ is the length of $\Gamma$. Here, $\mu$ and~$\nu$ are two positive constants and $|\Gamma|$ can be written as $\mathcal{H}^1(\Gamma)$, which is the 1-dimensional Hausdorff measure. The difficulty in studying~\eqref{Energy:MS} is that it involves two unknowns $J$ and $\Gamma$ of different natures: $J$ is a function defined on a 2-dimensional space, while $\Gamma$ is a 1-dimensional set. It is not easy to minimize the Mumford-Shah functional $E_{\mathrm{MS}}$ and the simplified Chan-Vese (CV) model~\cite{chan2001active} is proposed by using the piecewise constant assumption and thus the functional is reduced to:
\begin{equation}
	\label{Chan-Vese}
	\begin{aligned}
		E_{\mathrm{CV}}(c_1, c_2, \Omega)=\int_{\Omega_1}(I-c_1)^2dx  +\int_{\Omega\setminus\Omega_1}(I-c_2)^2dx + \nu|\partial \Omega_1|,
	\end{aligned}
\end{equation}
where $c_1,c_2$ are constants for foreground (fg) and background (bg) respectively and $\partial \Omega_1$ is the boundary of $\Omega_1$. Throughout this paper, we called $\Omega_1$ as the fg and $\Omega\setminus\Omega_1$ as the bg. Also, under the maximum a posterior (MAP) framework, the Chan-Vese model~\eqref{Chan-Vese} is derived in~\cite{cremers2007review} by assuming that fg/bg are random variables that generated from two Gaussian distributions and the prior distribution of the boundary is the length regularization term in~\eqref{Chan-Vese}. The Gaussian distribution assumptions for fg/bg are key factors for the success of segmentation but they do not hold for complex scenes. In Figure~\ref{problem}, we illustrate two typical cases: (i) The fg or bg does not satisfy the Gaussian distribution hypothesis (see the first row in Figure~\ref{problem} (b)); (ii) the distributions of fg and bg have been significantly overlapped (see the second row in Figure~\ref{problem} (b)).
The results of the CV model are present in Figure~\ref{problem}~(c). It is clear that the CV model fails in both cases due to the violation of its basic assumption. To widen the CV model's application range, there is a need to construct a more accurate model applicable to complex scenes.
\begin{figure}
	\centering
	\begin{tabular}{c@{\hspace{0.01\linewidth}}c@{\hspace{0.01\linewidth}}c@{\hspace{0.01\linewidth}}c@{\hspace{0.01\linewidth}}c@{\hspace{0.01\linewidth}}c}
		\includegraphics[width=0.15\linewidth,height=0.12\linewidth]{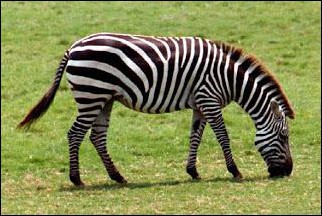} &
		\includegraphics[width=0.15\linewidth,height=0.12\linewidth]{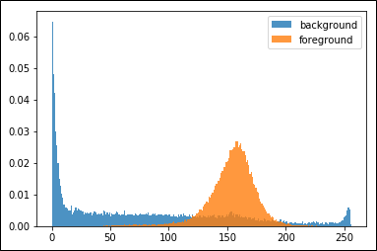} &
		\includegraphics[width=0.15\linewidth,height=0.12\linewidth]{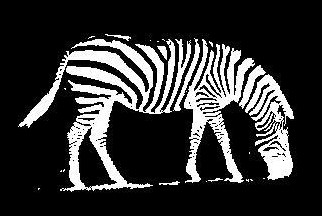} &
		\includegraphics[width=0.15\linewidth,height=0.12\linewidth]{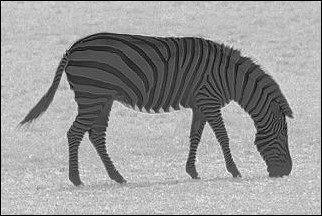} &
		\includegraphics[width=0.15\linewidth,height=0.12\linewidth]{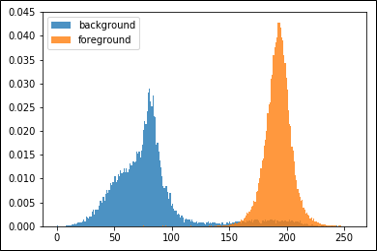} &
		\includegraphics[width=0.15\linewidth,height=0.12\linewidth]{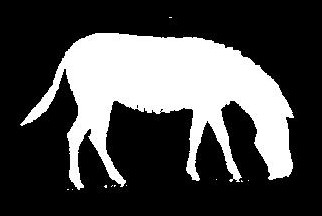} \\
		\includegraphics[width=0.15\linewidth,height=0.12\linewidth]{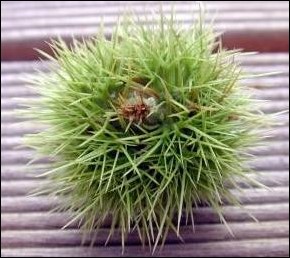} &
		\includegraphics[width=0.15\linewidth,height=0.12\linewidth]{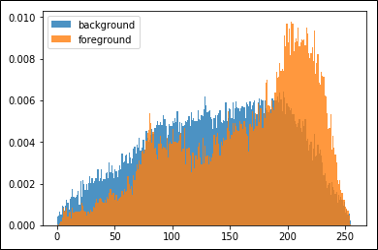} &
		\includegraphics[width=0.15\linewidth,height=0.12\linewidth]{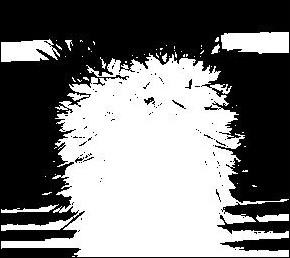}&
		\includegraphics[width=0.15\linewidth,height=0.12\linewidth]{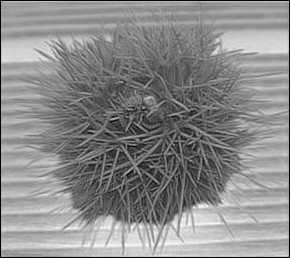} &
		\includegraphics[width=0.15\linewidth,height=0.12\linewidth]{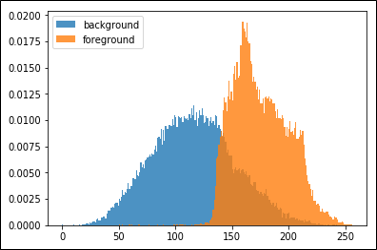} &
		\includegraphics[width=0.15\linewidth,height=0.12\linewidth]{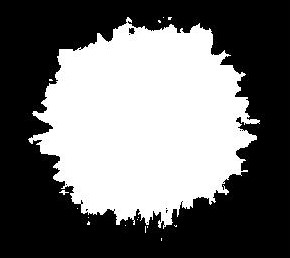}\\
		{\small{(a)}} & {\small{(b)}} & {\small{(c)}} & {\small{(d)}} & {\small{(e)}} & \small{(f)}
	\end{tabular}
	\caption{(a) Original image. (b) Foreground/background (fg/bg) distribution in image space. (c) Segmentation result by the CV model \eqref{Chan-Vese}. (d) Latent space representations of the original image. (e) Fg/bg distributions in latent spaces. (f) Our segmentation results.}
	\label{problem}
\end{figure}

In recent years, deep learning methods have achieved state-of-the-art performance in image segmentation~\cite{long2015fully,he2017mask}, which trains a deep neural network from a set of training samples. However, their performance heavily depends on a large number of high-quality training samples in which each pixel has a label. In practice, the labeling process is time-consuming~\cite{zhang2018collaborative} and need many human efforts, especially in domain-specific applications such as medical images and seismic data. To relax the supervision requirements, some recent works utilize weakly or semi-supervised image segmentation, e.g., relaxations of the pixel-level annotations to image level~\cite{papandreou2015weakly} or bound box level~\cite{dai2015boxsup}. These methods still need many training pairs to achieve good performance such that sufficiently many complex scenes are covered. Therefore, further relaxing the supervision requirements to the unsupervised setting has its own practical and scientific values and deserves to be studied. One recent unsupervised learning method uses deep image prior~\cite{ulyanov2018deep} for image recovery and decomposition problems. Yet, it suffers from the overfitting problem, and the result is inferior to traditional methods in some tasks, e.g., image denoising.
Besides the dataset dependency, the supervised networks may not be stable and suffer from adversarial perturbations, which cause false predictions~\cite{szegedy2013intriguing}. Meanwhile, the trained networks' generalization problem exists when the scene is complicated, and the object is not contained in the training set. More importantly, due to the "black box" property of deep neural networks, it is difficult to analyze the internal mechanism, especially for the failure cases. On the contrary, the traditional model based approaches have a clear and rigorous mathematical foundation. Thus, combining traditional methods with deep learning approaches and fully exploring both advantages is important for unsupervised image segmentation.
\begin{figure}
	\begin{center}
		\includegraphics[width=0.9\linewidth]{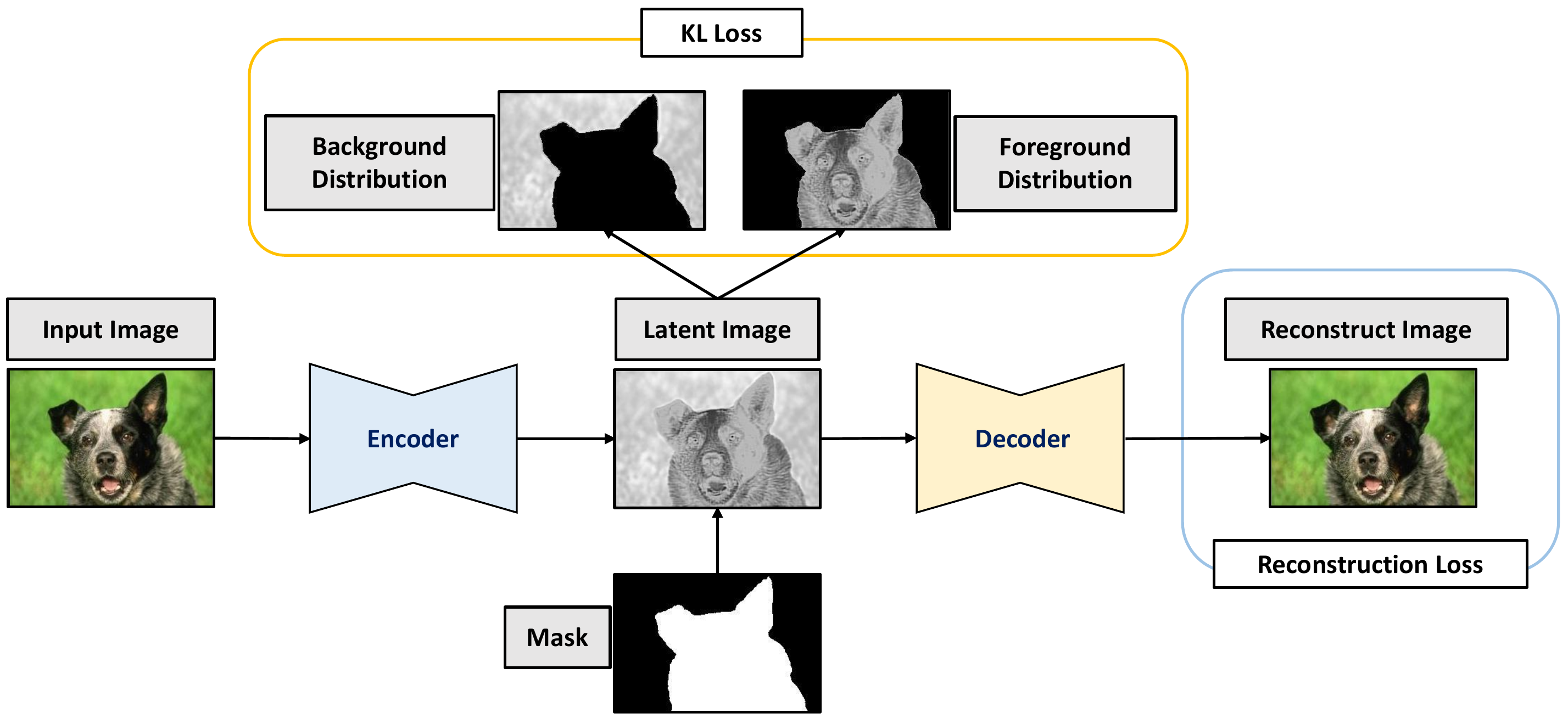}
	\end{center}
	\caption{The proposed unsupervised image segmentation framework. The input image is transformed by the encoder to the latent space, where the mask separates the latent image into the given foreground and background distribution. The decoder recovers the original image from the latent image.}
	\label{fig:workflow}
\end{figure}

\subsection{Summary of Contributions}
Motivated by the above analysis, our work aims at promoting the unsupervised image segmentation framework by combining the deep neural networks with the traditional maximum a posterior (MAP) framework. The basic idea is to map the input image to a latent image using a deep neural network. The fg/bg distributions are simplified in the latent space such that the CV model is applicable. In Figure~\ref{problem} (d)-(f), it shows the latent representations for fg/bg and validates our basic idea, which leads to more accurate segmentation results than CV model. 

Formally, we formulate our idea in a classic MAP framework that desires to maximize the posterior distribution. It is known that the posterior is proportion to the product of the likelihood and prior distribution. In the CV model~\eqref{Chan-Vese}, the prior distribution is an exponential function for the length of the segmentation curve. The likelihood assumes to be Gaussian distribution for fg and bg, which is not always accurate in practice. To address the above problem, we assume that each region corresponds to a latent variable that satisfies the Gaussian distribution and derive an evidence lower bound (ELBO) term to approximate the likelihood. The ELBO term is derived from the Variational Auto-Encoder (VAE)~\cite{kingma2013auto} that contains encoder and decoder networks. The above two networks build up the connections between image space and latent space and contains two terms: the Kullback-Leibler (KL) divergence loss and the reconstruction loss. More specifically, the KL loss represents the distance between the latent representations of foreground and background and the Gaussian distribution, while the reconstruction loss maintains the connection between image space and latent space. Figure~\ref{fig:workflow} demonstrates the workflow of our model. Replacing the likelihood in MAP with the derived ELBO, we can derive an unsupervised loss. Each term can be explained as model-driven approaches rather than designing by heuristics. Thus, our model does not need any pre-training or training samples except for the input image itself. We summarize our contributions as follows.

\begin{itemize}
	\item An unsupervised image segmentation method that integrates deep neural networks with the CV model is proposed by coupling each region with a latent variable. Using the expressive power of deep neural networks, the input image is mapped into a latent space in which the piecewise constant assumption for the fg and bg holds, thus significantly improving the CV model's segmentation accuracy. Also, a convergent numerical algorithm is proposed for solving the resulted non-convex and non-smooth optimization problem.
	\item Based on the MAP framework, we derive an ELBO term for approximating the likelihood, which naturally improves each term's explainability in our loss function and makes our network diagnosable. Moreover, this method's idea can be extended to multi-phase image segmentation and dataset based image segmentation problems.
	\item Extensive numerical experiments validate our method's robustness. They show that the proposed model improves the CV model's segmentation accuracy by a large margin and outperforms other unsupervised learning or classical approaches.
\end{itemize}

The rest of the paper is organized as follows. The related work is present in section~\ref{rel_work}. In section~\ref{proposed_method}, we derive the proposed model, which is based on the traditional CV model. We then give a numerical algorithm to solve this model and show that the proposed algorithm converges. In section~\ref{extensions}, we extend this idea to multi-phase segmentation, and dataset based segmentation tasks. In section~\ref{experiments}, we compare our two-phase segmentation method with other four algorithms~\cite{chan2001active,cremers2007review,gandelsman2018double,peng2016salient} on the Weizmann dataset~\cite{AlpertGBB07}, compare our multi-phase segmentation method with the Double-DIP~\cite{gandelsman2018double}, the Normal Cut method~\cite{shi2000normalized}, and compare our dataset based segmentation with the CV model~\cite{chan2001active} and one Generative Adversarial Network (GAN) based model ReDO~\cite{chen2019unsupervised}. In section~\ref{discussion}, we discuss the performance of the proposed method from six perspectives. Conclusions are given in section~\ref{conclusion}.

\section{Related Work}\label{rel_work}
The image segmentation problem is important, and there are many works in this direction. In this section, we briefly review the papers that are closely related to this work.

\noindent{\bf{Active contour-based segmentation methods.}} An early attempt in image segmentation is the Snakes model, which used PDEs to find the optimal segmentation curve~\cite{kass1988snakes}. The Snakes model belongs to edge models, which heavily rely on the choice of initial curves. The level-set method~\cite{sethian1996fast} is an alternative method for representing the segmentation region and can handle the topological change of the segmentation curve. The geodesic active contours method is proposed in~\cite{caselles1993geometric}, but it is sensitive for initialization and has many local minima. The Chan-Vese model~\cite{chan2001active} is a region-based segmentation model, and the statistical understandings for these region-based models are given in~\cite{cremers2007review}. To avoid from getting stuck in local minima, convex relaxation approaches~\cite{bresson2007fast,pock2009convex,chan2006algorithms,cai2013two} and the graph cut method~\cite{grady2008reformulating} are proposed for improving the convergence behavior. The LBF model~\cite{li2008minimization} is proposed for intensity inhomogeneous images. Convex shape prior is introduced in~\cite{chan2005level,luo2019convex}. Unlike the above work, our method is a VAE based method that takes advantage of deep networks' high expressive power to model complex scenes, leading to better segmentation results.

\noindent{\bf Variational Auto-Encoders (VAEs).} VAE is an important generative model based on variational inference~\cite{kingma2013auto}. A lower bound of the original distribution is estimated by introducing a latent space, which contains random variables following simple distribution. Deep neural networks parameterize the relationship between latent space and original space. There have been many signs of progress along this direction in image processing, such as image denoising~\cite{prakash2020fully}, image compression~\cite{zhou2018variational}, and image super-resolution~\cite{wang2019variational}. However, these methods need a large number of training samples to approximate accurate distribution. Our work is an early attempt at using VAEs for unsupervised single image segmentation to the best of our knowledge.

\noindent{\bf Deep image prior.} Deep image prior~\cite{ulyanov2018deep} explicitly represents the real images by deep neural networks with low dimensional random input. It is shown that deep image prior can be applied for many image processing tasks such as image denoising, inpainting, and super-resolution. Since deep image prior does not need external learning, it inspires many follow-ups works. Gandelsman et al.~\cite{gandelsman2018double} propose a double deep image prior method which uses two deep neural networks with different regularization for representing two layers for given images. It shows good performance in many image decomposition tasks, including segmentation, dehazing, and transparent layer separation. However, the objective function in these methods is heuristically designed and lacks mathematical understanding. Instead, our model is constructed from the traditional MAP framework, which improves each term's interpretability in the object function.

\section{The Proposed Approach}\label{proposed_method}
In this section, we firstly review the statistical model for segmentation~\cite{cremers2007review} and then propose our method for image segmentation based on VAEs. For the sake of simplicity, we represent the $I$ as a matrix in $\mathbb{R}^{n\times m}$, and denote $I_x=I(x)$, where $x$ runs for the set of matrix indexes $\Omega=\{(i,j)\mid i=1,\dots,n;j=1,\dots,m\}$.

\subsection{Statistical Model for Image Segmentation}
Given an image $I \in \mathbb{R}^{n\times m}$, the segmentation problem can be formulated to maximize the posterior probability $p(\mathcal{P}(\Omega)\mid I)$ where $\mathcal{P}(\Omega)$ denotes a partition of $\Omega$, i.e., there exist $N$ non-overlapped sub-regions $\Omega_i$ such that $\Omega=\cup_{i=1}^N\Omega_i$. Using the Bayesian rule, the posterior is 
\begin{equation}\label{model:deepCV}
	p(\mathcal{P}(\Omega)\mid I)\propto p(I\mid \mathcal{P}(\Omega))p(\mathcal{P}(\Omega)),
\end{equation}
where $p(I\mid \mathcal{P}(\Omega))$ is the likelihood and $p(\mathcal{P}(\Omega))$ is the prior information. Specifically, in binary segmentation, there are only two phases $\Omega_1$, $\Omega_2$, we introduce a binary valued segmentation mask $u\in\{0,1\}^{m\times n}$ to represent the regions, i.e.,
\begin{equation}
	\Omega_1 = \{x\in\Omega\mid u_x=1\}, \quad \Omega_2 = \{x\in\Omega\mid u_x=0\}.
\end{equation}
Thus, the model \eqref{model:deepCV} is reduced to \begin{equation}\label{model:graph}
    p(u \mid I) \propto p(I \mid u)p(u).
\end{equation} 
The term $p(u)$ is the prior, and there are many types of prior knowledge of the segmentation curves such as shape and length~\cite{luo2019convex}. In CV model~\cite{chan2001active}, $p(u)$ is chosen as
\begin{equation}\label{prior-cv}
	p(u)  = \prod_{x\in\Omega} p(u_x),\quad p(u_x)
	\propto \exp(- \nu \|\nabla u_x\|_2),
\end{equation}
where $\|\nabla u_{i,j}\|_2 = \sqrt{(u_{i,j}-u_{i-1,j})^2+(u_{i,j}-u_{i,j-1})^2}$. Under the independent assumption of each region and pixels, the  first term in \eqref{model:graph} is
\begin{equation}\label{likelihood_1}
    p(I\mid u) = \prod_{x\in\Omega}p(I_x\mid u_x).
\end{equation}
Let $\mathcal{N}(0,1)$ to be the standard Gaussian distribution and set 
\begin{equation}\label{Gaussian_ass}
    p(I_x\mid u_x) = 
    \begin{cases}
    \mathcal N(c_1,1), & \mbox{ if } u_x = 1, \\
    \mathcal N(c_2,1), & \mbox{ if } u_x = 0,
    \end{cases}
\end{equation}
combing~\eqref{Gaussian_ass} with~\eqref{prior-cv}, we obtain the energy function of the discretized CV model
\begin{equation}
	\label{Chan-Vese-LevelSet-alt}
	\begin{aligned}
		E_{\mathrm{CV}}(c_1, c_2, u)&=\sum_{x\in\Omega} \nu \|\nabla u_x\|_2 + u_x(I_x-c_1)^2 + (1-u_x)(I_x-c_2)^2.
	\end{aligned}
\end{equation}
It is clear that the Gaussian assumption~\eqref{Gaussian_ass} does not always hold in practice, and thus one extended version of the CV model is imposing more complex assumptions on $I_x\mid u_x$ such as the mixture of Gaussian~\cite{fu2012color}. However, it is not easy to find a universal method for modeling the likelihood accurately in practical applications. To address this problem, we apply the deep neural network-based generative models that approximate the likelihood term $p(I\mid u)$. 

\subsection{The Deep CV Model}\label{deep_cv}
In this subsection, we propose an unsupervised image segmentation framework by using the VAE~\cite{kingma2013auto}, which is a powerful generative model. To approximate the likelihood $p(I \mid u)$, we associate $I$ with a latent variable $Z \in \mathbb{R}^{n\times m\times d}$, where $d$ is the dimension of the latent space, and impose the Gaussian assumptions on the latent variable~$Z$, i.e., \
\begin{equation}\label{ass1}
	p(Z\mid u)=\prod_{x\in\Omega}p(Z_x\mid u_x), \quad 
	 p(Z_x\mid u_x) = 
    \begin{cases}
    \mathcal N(\mu_1, \Sigma_1), & \mbox{ if } u_x = 1, \\
    \mathcal N(\mu_2, \Sigma_2), & \mbox{ if } u_x = 0,
    \end{cases}
\end{equation} 
where $\mu_i \in \mathbb{R}^d$ is the first order moment, and $\Sigma_i$ is the covariance matrix. Moreover, the likelihood $p(I\mid u)$ has the relationship:
\begin{equation}\label{LowerBound}
	\begin{aligned}
		\ln p(I\mid u) = \ln \int p(I, Z\mid u) dZ 
		& =  \ln \int \frac{p(I, Z\mid u)}{q(Z\mid I)} q(Z\mid I) dZ\\
		& \geq \int \ln\left(\frac{p(I, Z\mid u)}{q(Z\mid I)}\right) q(Z\mid I) dZ,
	\end{aligned}
\end{equation}
where the last inequality follows from the Jensen's inequality. Define the evidence of lower bound (ELBO) as the right hand side in \eqref{LowerBound}, i.e.\ 
\begin{equation}
	\mathrm{ELBO} = \int \ln\left(\frac{p(I, Z\mid u)}{q(Z\mid I)}\right) q(Z\mid I) dZ,
\end{equation}
it provides a lower bound for the log-likelihood $\ln p(I\mid u)$ which is originally given in~\cite{kingma2013auto}. By direct calculation, the ELBO term is equal to
\begin{equation}\label{ELBO}
	E_{q(Z \mid I)}\ln (p(I\mid Z,u))-\mathrm{KL}(q(Z\mid I)\|p(Z\mid u)),
\end{equation}
where $\mathrm{KL}$ is the Kullback–Leibler (KL) divergence. Choosing 
\begin{equation}\label{ass2}
	\begin{aligned}
		&p(I\mid Z,u) = \prod_{x\in\Omega} p(I_x\mid Z) = \prod_{x\in\Omega} \mathcal{N}(\mathcal{F}(Z)_x, 1), \\
		&q(Z\mid I) = \prod_{x\in\Omega} q(Z_x\mid I) = \prod_{x\in\Omega} \mathcal{N}(\mathcal{G}^{\mu}(I)_x,\mathcal{G}^{\sigma}(I)_x),
	\end{aligned}
\end{equation}
where $\mathcal{F}$ and $\mathcal{G}=\left(\mathcal{G}^{\mu},\mathcal{G}^{\sigma}\right)$ are the decoder map and encoder map, respectively. The next proposition shows that~\eqref{ELBO} has an analytical form.
\begin{prop}\label{MainProp}
	Let $u$ be a binary function on $\Omega$, we set $\Omega_1 = \{x\mid u_x = 1\}$ and $\Omega_2=\{x\mid u_x = 0\}$. If $Z$ satisfies~\eqref{ass1} and choose $p(I\mid Z,u)$ and $q(Z\mid I)$ as~\eqref{ass2}, then the~{\rm{ELBO}} in~\eqref{ELBO} is equal to
	\begin{equation}\label{ELBO1}
	    -\frac{1}{2} E_{\eta}\|\mathcal{F}(\mathcal{G}^{\mu}(I)+ \sqrt{\mathcal{G}^{\sigma}(I)}\eta)-I\|^2 - \sum_{x\in \Omega} \left( u_x\mathrm{KL}_x^{\Omega_1} + (1-u_x)\mathrm{KL}_{x}^{\Omega_2} \right) + c
	\end{equation}
	where $\eta\sim\mathcal{N}(0,\mathbf{I})$, $c$ is a constant, and for $i=1,2$,
    \begin{equation}
        \mathrm{KL}_{x}^{\Omega_i} = \frac{1}{2} \left(\ln \frac{|\Sigma_i|}{|\mathcal{G}^{\sigma}(I)_x|} - d + \operatorname{tr} (\Sigma_i^{-1}\mathcal{G}^{\sigma}(I)_x) + (\mathcal{G}^{\mu}(I)_x - \mu_i)^{T} \Sigma_i^{-1} (\mathcal{G}^{\mu}(I)_x - \mu_i)\right).
    \end{equation}
\end{prop}
The proof of Proposition~\ref{MainProp} is shown in section~\ref{proof_of_main_prop} in Appendix. In our model, we replace the likelihood term $\ln p(I\mid u)$ by the ELBO term derived in Proposition~\ref{MainProp} and obtain our loss function as
\begin{equation}\label{Energy_Functional}
	E(u,\mathcal{F},\mathcal{G})= \underbrace{\frac{1}{2}  E_{\eta}\|\mathcal{F}(\mathcal{G}^{\mu}(I)+ \sqrt{\mathcal{G}^{\sigma}(I)}\eta)-I\|^2}_{\mathrm{Reconstruction}} + \sum_{x \in \Omega} \underbrace{\nu \|\nabla u_x\|_2}_{\mathrm{Regularization}} + \underbrace{u_x\mathrm{KL}_x^{\Omega_1} + (1-u_x)\mathrm{KL}_{x}^{\Omega_2}}_{\mathrm{KL}}.
\end{equation}
In this paper, we parameterize the decoder $\mathcal{F}$ and the encoder $\mathcal{G}$ by deep neural networks, and the concrete settings of $\mathcal{F}$ and $\mathcal{G}$ are given in section~\ref{experiments}.

\subsection{Numerical Algorithm}\label{numerical_algo}
The minimization of~\eqref{Energy_Functional} is a challenging problem due to its non-convexity and non-smoothness. Similar to~\cite{chan2006algorithms,potts1952some}, we formulate~\eqref{Energy_Functional} as the following binary labeling problem,
\begin{equation}\label{model:relax}
	E (\theta,\gamma,u) = \nu\|\nabla u\|_{1,2} + \frac{1}{2}E_\eta\|\mathcal{F}_\theta(\mathcal{G}^{\mu}_\gamma(I)+\sqrt{\mathcal{G}^{\sigma}_{\gamma}(I)}\eta)-I\|^2 + \mathrm{KL}^{\Omega_1}_{u} + \mathrm{KL}^{\Omega_2}_{1 - u}, \quad u\in\{0,1\}^{n\times m},
\end{equation}
where $\mathcal{F}$ and $\mathcal{G}$ are parameterized by deep neural networks $\mathcal{F}_\theta$ and $\mathcal{G}_\gamma$, $\theta$ and $\gamma$ are parameters, $\nabla v_{i,j} = (v_{i,j}-v_{i-1,j},v_{i,j}-v_{i,j-1})^\top\in \{-1,0,1\}$, $\|\nabla v\|_{1,2} = \sum_{x\in \Omega} \|\nabla v_{x}\|_2$, and $\mathrm{KL}_{u}^{\Omega_i} = \sum_{x\in \Omega} u_x \mathrm{KL}_{x}^{\Omega_i}$, $i=1,2$.
To relax the combinatorial constraints on $u$, we approximate $u$ by $S(\phi)$ where $S(x) = 1/(1+\exp(-x))$ is the sigmoid function which is a one-to-one mapping between $\mathbb{R}$ and $(0,1)$. Thus, the model~\eqref{model:relax} is relaxed to
\begin{equation}\label{loss}
	E (\theta,\gamma,\phi) = \nu\|\nabla S(\phi)\|_{1,2} + \frac{1}{2} E_\eta\|\mathcal{F}_\theta(\mathcal{G}^{\mu}_\gamma(I)+ \sqrt{\mathcal{G}^{\sigma}_{\gamma}(I)}\eta)-I\|^2+\mathrm{KL}^{\Omega_1}_{S(\phi)}+\mathrm{KL}^{\Omega_2}_{1 - S(\phi)}.
\end{equation}
Moreover, we introduce a squared $\ell_2$-term for penalizing the $\nabla S(\phi)$ in~\eqref{loss}, which leads to the following loss function:
\begin{equation}\label{relax:loss}
	\begin{aligned}
		E_{\LS}(\theta,\gamma, \phi, w) =& \nu\|w\|_{1,2} +\frac{\lambda}{2}\|w-\nabla S(\phi)\|_2^2+ \frac{1}{2} E_\eta\|\mathcal{F}_\theta(\mathcal{G}^{\mu}_\gamma(I)+ \sqrt{\mathcal{G}^{\sigma}_{\gamma}(I)}\eta)-I\|^2 \\
		&+\mathrm{KL}^{\Omega_1}_{S(\phi)}+\mathrm{KL}^{\Omega_2}_{1 - S(\phi)},
	\end{aligned}
\end{equation}
where $\lambda>0$ is the penalty parameter. 

To minimize~\eqref{relax:loss}, we adopt the alternating optimization method. Specifically, given the current estimate $(\theta^k,\gamma^k,\phi^k,w^k)$, our method consists of the following two steps: fixing $w^k$, update $(\theta^{k+1},\gamma^{k+1},\phi^{k+1})$ via
\begin{equation}\label{update}
	\begin{cases}
		\theta^{k+1} = \theta^k - \alpha_1\nabla_{\theta} E_{\LS}(\theta^k,\gamma^k,\phi^k,w^k)\\
		\gamma^{k+1} = \gamma^k - \alpha_2\nabla_{\gamma} E_{\LS}(\theta^k,\gamma^k,\phi^k,w^k)\\
		\phi^{k+1} = \phi^k - \alpha_3\nabla_{\phi} E_{\LS}(\theta^k,\gamma^k,\phi^k,w^k)
	\end{cases},
\end{equation}
where $\alpha_i>0, i=1,2,3$ are step sizes; fixing $(\theta^{k+1},\gamma^{k+1},\phi^{k+1})$, update $w^{k+1}$ via
\begin{equation}\label{sub:w}
	w^{k+1} = \arg\min_w \left\{\nu\|w\|_{1,2}+\frac{\lambda}{2}\|w-\nabla S(\phi^{k+1})\|_2^2\right\}.
\end{equation}
It is noted that $\nabla_\theta E_{\LS}$ and $\nabla_\gamma E_{\LS}$ can be efficiently estimated by the auto-differentiation scheme in open source software, e.g., TensorFlow~\cite{abadi2016tensorflow} or PyTorch~\cite{paszke2019pytorch}. Moreover, we give the analytic solution for the sub-problem~\eqref{sub:w} in the next proposition, and the proposed unsupervised segmentation is present in Algorithm~\ref{alg:VAE_seg}.
\begin{prop}\label{cal_length}
	The solution of \eqref{sub:w} is given by
	\begin{equation}\label{sol:w}
		w^{k+1}_{i,j} = \max\left(\|\nabla S(\phi^{k+1})_{i,j}\|_2-\frac{\nu}{\lambda},0\right)\frac{\nabla S(\phi^{k+1})_{i,j}}{\|\nabla S(\phi^{k+1})_{i,j}\|_2}, \forall i,j=1,\ldots,n.
	\end{equation}
\end{prop}
\begin{algorithm}[] 
	\caption{The proposed unsupervised segmentation} 
	\label{alg:VAE_seg} 
	\begin{algorithmic}[1]
		\Require 
		Image $I:\Omega \to \mathbb{R}$ needs to be segmented.
		\Ensure
		The foreground and background region mask $u: \Omega \to \{0,1\}$.
		\State Initial encoder, decoder network $\mathcal{F}_{\theta}, \mathcal{G}_{\gamma}$ and 
		the region function $\phi$.
		\For {$k=0,1,2,3,...$}
		\State Update network parameters $\gamma, \theta$ and function $\phi$ by~\eqref{update}.
		\State Update $w$ by~\eqref{sol:w}.
		\EndFor
		\State Take the sign of $\phi$ to get the final region mask $u$. \\
		\Return $u$.
	\end{algorithmic} 
\end{algorithm}

\begin{rem}
	The reconstruction loss in~\eqref{Energy_Functional} can be efficiently estimated by the Monte-Carlo method~\cite{metropolis1949monte}. In this work, the number of sampling is set as 1. Moreover, the formulation can be easily extended for the color image by setting $I_x\in\mathbb{R}^3$, and the results are shown in section~\ref{experiments}.
\end{rem}

\begin{rem}
	In this work, we choose the decoder map $\mathcal{F}$ and the encoder map $\mathcal{G}$ as the classic U-nets~\cite{ronneberger2015u}.
	Specifically, a U-net $h$ can be represented by 
	\begin{equation}\label{eq:U-nets}
		h(x) = H_K\circ H_{K-1} \circ \cdots \circ H_1(x),
	\end{equation}
	and $H_i(x) = \sigma \circ P\circ A(x)$ where $\sigma$ is the nonlinear activation function, $P$ is the identity, downsampling or upsampling operator, $A$ is the linear operator that represents the convolution layer or fully connected layer. 
\end{rem}

\subsection{Convergence analysis}
In this subsection, we establish the convergence properties for Algorithm~\ref{alg:VAE_seg} based on the analysis of the multi-block iteration scheme~\cite{bao2015dictionary,bolte2014proximal,xu2013block} for solving non-convex minimization problems. Before proceeding to the analysis, we make the following assumptions on $E_{\LS}$.
\begin{ass}\label{assum1}
	There exists a bounded set $\mathcal M$ such that the sequence $ \{(\theta^k,\gamma^k,\phi^k,w^k)\}\subset\mathcal{M}$. Moreover, there exists $L_M$ such that for all $(\theta_i,\gamma_i,\phi_i,w)\in\mathcal{M}, i=1,2$, we have
	\begin{equation}\label{Lipschitz}
		\|\nabla_{\theta,\gamma,u} E_{\LS}(\theta_1,\gamma_1,\phi_1,w)-\nabla_{\theta,\gamma,u} E_{\LS}(\theta_2,\gamma_2,\phi_2,w)\|\leq L_M\|(\theta_1,\gamma_1,\phi_1)-(\theta_2,\gamma_2,\phi_2)\|.
	\end{equation}
\end{ass}
The boundedness of the sequence can be guaranteed by imposing the physical constraints on $\theta,\gamma,\phi$, which is a common clip operation in network training. Following the convergence proof framework proposed in~\cite{bolte2014proximal}, the next theorem establishes the convergence property of Algorithm~\ref{alg:VAE_seg}.
\begin{thm}\label{Thm:Seqconver}
	Suppose Assumption~\ref{assum1} holds and $\mathcal{F}$, $\mathcal{G}$ are deep neural networks defined in~\eqref{eq:U-nets} with sigmoid activation. Let $\{x^k\}=\{(\theta^k,\gamma^k,\phi^k,w^k)\}$ be the sequence generated by Algorithm~\ref{alg:VAE_seg} with step sizes $\alpha_i \leq 2/L_{M}$, $i=1,2,3$ in~\eqref{update}. Then, the sequence $\{x^k\}$ converges to a stationary point of $E_{\LS}$ defined in~\eqref{relax:loss}. 
\end{thm}
\begin{proof}
	The proof is present in Section~\ref{proof_main_thm} in Appendix.
\end{proof}
\begin{rem}
	The ReLU activation can be replaced by the other differentiable activation functions, e.g., CReLU. It is noted that the softplus function can be seen as the smooth approximation of the ReLU, which is a common activation used in deep learning.
\end{rem}

\section{Extensions}\label{extensions}
In this section, we extend our idea to multi-phase segmentation and dataset based segmentation problems.
\subsection{Multi-phase Image Segmentation}
Assume there are $N$ phases $\Omega_1, \Omega_2, ..., \Omega_N$, denote the indicator function of $\Omega_i$ as $u_i$, i.e., $u_i(x)=\begin{cases}
	1, x\in \Omega_i \\
	0, x\notin \Omega_i
\end{cases}$ and $\sum_{i=1}^N u_i = 1$ for disjoint regions. Then the loss function in~\eqref{Energy_Functional} becomes
\begin{equation}\label{Energy_Functional_Multi-phase}
	E(\{u_i\}_{i=1}^{N},\mathcal{F},\mathcal{G}) = \underbrace{\frac{1}{2} E_{\eta}\|\mathcal{F}(\mathcal{G}^{\mu}(I)+ \sqrt{\mathcal{G}^{\sigma}(I)}\eta)-I\|^2}_{\mathrm{Reconstruction}} + \sum_x \underbrace{\nu\sum_{i=1}^N\|\nabla u_{ix}\|_2}_{\mathrm{Regularization}}+
	\underbrace{\sum_{i=1}^N u_{ix} \mathrm{KL}^{\Omega_i}_x}_{\mathrm{KL}},
\end{equation}
where $\eta\sim\mathcal{N}(0,\mathbf{I})$. Instead of directly using $N$ independent level-set functions~\cite{zhao1996variational} or the vector valued Heaviside function~\cite{vese2002multiphase}, we represent phases $(u_1, u_2, \ldots, u_N)$ by $\Phi=(\phi_1, \phi_2, \ldots,\phi_N)$ with 
\begin{equation}
	u_i = S_i(\Phi) = \frac{\exp{\phi_i}}{\sum_{j=1}^{N}\exp{\phi_j}},\ i = 1,2,\ldots,N.
\end{equation}

The smoothed relaxation of \eqref{Energy_Functional_Multi-phase} is
\begin{equation}\label{Energy_Functional_Multiphase_smooth}
	E(\Phi,\mathcal{F},\mathcal{G})= \frac{1}{2} E_{\eta}(\mathcal{F}\|\mathcal{G}^{\mu}(I)+ \sqrt{\mathcal{G}^{\sigma}(I)}\eta)-I\|^2+\sum_x \nu\sum_{i=1}^N\|\nabla S_{i}(\Phi)_x\|_2+ \sum_{i=1}^N S_{i}(\Phi)_x\mathrm{KL}^{\Omega_i}_x.
\end{equation}
Same as section~\ref{numerical_algo}, the object function of our multi-phase segmentation model becomes
\begin{equation}\label{model:relax_multiphase}
	\tilde{E} (\theta,\gamma,\Phi) = \frac{1}{2} E_\eta\|\mathcal{F}_\theta(\mathcal{G}^{\mu}_\gamma(I)+\sqrt{\mathcal{G}^{\sigma}_{\gamma}(I)}\eta)-I\|^2 +  \sum_{i=1}^N\nu\|\nabla S_i(\Phi)\|_{1,2}
	+ \sum_{i=1}^N \mathrm{KL}^{\Omega_i}_{S(\phi)},
\end{equation}
with introducing a squared $\ell_2$-term for penalizing $\nabla S_i(\Phi)$, we obtain the following loss function
\begin{equation}\label{model:loss_multiphase}
	\begin{aligned}
		E_{\LS} (\theta,\gamma,\Phi) =& \sum_{i=1}^N\left(\nu\|w_i\|_{1,2}+\frac{\lambda}{2}\|w_i-\nabla S_i(\Phi)\|_2^2\right) \\
		&+ \frac{1}{2} E_\eta\|\mathcal{F}_\theta(\mathcal{G}^{\mu}_\gamma(I)+\sqrt{\mathcal{G}^{\sigma}_{\gamma}(I)}\eta)-I\|^2 + \sum_{i=1}^N \mathrm{KL}^{\Omega_i}_{S(\phi)},
	\end{aligned}
\end{equation}
where $\lambda > 0$ is the penalty parameter. Denote $(w_1^k,...,w_N^k)$ by $W^k$, with the same alternative optimization method in section~\ref{numerical_algo}, we update $(\theta^{k+1},\gamma^{k+1},\Phi^{k+1})$ via
\begin{equation}\label{update_multiphase}
	\begin{cases}
		\theta^{k+1} = \theta^k - \alpha_1\nabla_{\theta} E_{\LS}(\theta^k,\gamma^k,\Phi^k,W^k)\\
		\gamma^{k+1} = \gamma^k - \alpha_2\nabla_{\gamma} E_{\LS}(\theta^k,\gamma^k,\Phi^k,W^k)\\
		\Phi^{k+1} = \Phi^k - \alpha_3\nabla_{\Phi} E_{\LS}(\theta^k,\gamma^k,\Phi^k,W^k)
	\end{cases},
\end{equation}
where $\alpha_i>0, i=1,2,3$ are step sizes; fixing $(\theta^{k+1},\gamma^{k+1},\Phi^{k+1})$, update $W^{k+1}$ via
\begin{equation}\label{sub:w_multiphase}
	w_i^{k+1} = \arg\min_{w_i} \left\{\nu\|w_i\|_{1,2}+\frac{\lambda}{2}\|w_i-\nabla S_{i}(\Phi^{k+1})\|_2^2\right\},\quad \forall i=1,\ldots,N.
\end{equation}
In particular, from Proposition~\ref{cal_length} we have:
\begin{equation}\label{sol:w_multiphase}
	w^{k+1}_{i\text{ } k,l} = \max\left(\|\nabla S_{i}(\Phi^{k+1})_{k,l}\|_2-\frac{\nu}{\lambda},0\right)\frac{\nabla S_{i}(\Phi^{k+1})_{k,l}}{\|\nabla S_{i}(\Phi^{k+1})_{k,l}\|_2}, \forall i=1,\ldots,N; \text{ } k,l=1,\ldots,n.
\end{equation}
In summary, the proposed multi-phase segmentation algorithm is present in Algorithm~\ref{alg:VAE_multi_seg}.
\begin{algorithm}
	\caption{The proposed unsupervised multi-phase segmentation} 
	\label{alg:VAE_multi_seg} 
	\begin{algorithmic}[1]
		\Require 
		Image $I:\Omega \to \mathbb{R}$ needs to be segmented and $N$ the number of regions
		\Ensure
		The multi-phase region mask $u: \Omega \to \{1,\dots,N\}$.
		\State Initial encoder, decoder network $\mathcal{F}_{\theta}$, $\mathcal{G}_{\gamma}$ and 
		the region function $\Phi=(\phi_1,\dots,\phi_N)$.
		\For {$k=0,1,2,3,...$}
		\State Update network parameters $\gamma$, $\theta$ and function $\Phi$ by~\eqref{update_multiphase}.
		\State Update $w$ by~\eqref{sol:w_multiphase}.
		\EndFor
		\State Let $u(x) = \arg \max_{i}\{\phi_i(x)\mid i=1,\dots,N\}$. \\
		\Return $u$.
	\end{algorithmic} 
\end{algorithm}

\subsection{Dataset Based Image Segmentation}
In this section, we extend the proposed idea to train a deep neural network on a set of unlabeled images. Once the training process is finished, the network is fixed for estimating the segmentation masks on the test images, and thus significantly accelerates the inference speed comparing to classical single image based methods. Mathematically, given an image dataset $\mathcal{S}=\{I_i\}$ sampled from the image distribution $p(I)$, e.g., image of flowers, human faces. Our goal is to learn a segmentation function $\mathcal{U}$ such that it can separate each input image $I_i$ into two disjoint parts: the fg $\Omega_1^i$ and the bg $\Omega_2^i$. To achieve this goal, we propose our dataset based segmentation objective function by modifying the loss function in~\eqref{Energy_Functional} to
\begin{equation}\label{Energy_Functional_dataset}
\begin{aligned}
    E(\mathcal{U},\mathcal{F},\mathcal{G}) =& \sum_{I \in \mathcal{S}} \underbrace{\frac{1}{2} E_{\eta}(\mathcal{F}\|\mathcal{G}^{\mu}(I)+ \sqrt{\mathcal{G}^{\sigma}(I)}\eta)-I\|^2}_{\mathrm{Reconstruction}} \\
    &+ \underbrace{\sum_{x\in \Omega} \mathcal{U}(I)_{x} \mathrm{KL}_{x}^{\Omega_1} + \left(1-\mathcal{U}(I)_{x}\right) \mathrm{KL}_{x}^{\Omega_2}}_{\mathrm{KL}} + \underbrace{\mathcal{R}(\mathcal{U}(I))}_{\mathrm{Regularization}},
\end{aligned}
\end{equation}
where $\eta \sim \mathcal{N}(0, \mathbf{I})$, and $\mathcal{R}$ denotes the regularization functions. Compared to the previous models in single image case, we use a deep neural network that outputs the segmentation results. Thus, we impose additional  constraints on $\mathcal{U}$ such that it is applicable for the images in the whole dataset.

\noindent{\bf{Augmentation invariant.}} The segmentation function $\mathcal{U}$ is exchangeable with the augmentation operators. Denote these operators as $\mathcal{O}$, for each iteration, we rotate the mini-batch images $\mathcal{S}_{k}$ for $90^{\circ}$, or $180^{\circ}$, or $270^{\circ}$, and flip the images randomly, then compute the binary cross entropy loss between $\mathcal{U} \left(\mathcal{O}(\mathcal{S}_{k}) \right)$ and $\mathcal{O} \left(\mathcal{U}(\mathcal{S}_{k})\right)$, i.e.,
	\begin{equation}\label{Aug_loss}
		BCE(\mathcal{U}) = -\sum_{I\in \mathcal{S}_{k}} \sum_{x} \mathcal{O} \left(\mathcal{U}(I)\right)_{x} \ln  \mathcal{U} \left(\mathcal{O}(I)\right)_{x}
		- \left(1 - \mathcal{O} \left(\mathcal{U}(I)\right)_{x}\right) \ln \left(1 - \mathcal{U} \left(\mathcal{O}(I)\right)_{x} \right).
	\end{equation}

\noindent{\bf{Conservation of region information.}} To avoid generating empty regions, we use a discriminator network $\mathcal{D}$ to distinguish whether the outputs of decoder $\mathcal{F}$ are generated from empty regions. For each iteration, we sample mini-batch images $\mathcal{S}_{k}$ from dataset $\mathcal{S}$, and sample fake empty fg/bg images $\mathcal{S}_{\Omega_1}/\mathcal{S}_{\Omega_2}$ from decoder $\mathcal{F}$ by setting $\mathcal{S}_{\Omega_i} = \mathcal{F}(Z_{\Omega_i})$, where $Z_{\Omega_i} \sim \mathcal{N}(\mu_i, \Sigma_i)$, $i=1,2$. Then we compute the classification binary cross entropy loss for discriminator $\mathcal{D}$ to distinguish real images $\mathcal{S}_{k}$ from fake fg and bg images $\mathcal{S}_{\Omega_1}$, $\mathcal{S}_{\Omega_2}$, i.e.,
	\begin{equation}\label{loss_for_d}
		BCE(\mathcal{D}) = -\sum_{I \in \mathcal{S}_{k}} \ln \mathcal{D}(I) - \sum_{I \in \mathcal{S}_{\Omega_1}} \ln (1 - \mathcal{D}(I)) - \sum_{I \in \mathcal{S}_{\Omega_2}} \ln (1 - \mathcal{D}(I)).
	\end{equation}
	Then we fix discriminator $\mathcal{D}$ and decoder $\mathcal{F}$ compute the conservation of region information loss for segmentation function $\mathcal{U}$,
	\begin{equation}\label{CRI_loss}
		CRI(\mathcal{U}) = -\sum_{I\in \mathcal{S}_{k}}\ln \mathcal{D}\left(\mathcal{F}\left(Z_{\Omega_1} \odot \mathcal{U}(I) + Z_{\Omega_2} \odot \left(1 - \mathcal{U}(I) \right) \right )\right)
	\end{equation}
	which force $\mathcal{U}$ generate non-empty regions.

Thus, the regularization loss $\mathcal{R}(\mathcal{U}(I))$ in \eqref{Energy_Functional_dataset} is the summation of the augmentation invariant loss~\eqref{Aug_loss} and the conservation of region information loss~\eqref{loss_for_d}-\eqref{CRI_loss}. Moreover, we adopt the alternating minimization method for solving \eqref{Energy_Functional_dataset}. For each iteration, we first update $\mathcal{F}$, $\mathcal{G}$, $\mathcal{U}$ with the reconstruction and KL losses in~\eqref{Energy_Functional_dataset}, then we update $\mathcal{U}$ with the regularization loss. See Algorithm~\ref{alg:VAE_dataset_seg} for the details.

\begin{algorithm}
	\caption{The proposed dataset based segmentation} 
	\label{alg:VAE_dataset_seg} 
	\begin{algorithmic}[1]
		\Require 
		Image dataset $\mathcal{S} = \{I_i\}$
		\Ensure
		The segmentation function $\mathcal{U}$.
		\State Initial encoder, decoder network $\mathcal{F}, \mathcal{G}$, segmentation function $\mathcal{U}$, and discriminator $\mathcal{D}$.
		\For {$k=0,1,2,3,...$}
		\State Sample a mini-batch $\mathcal{S}_{k}$ from $\mathcal{S}$.
		\State Update network $\mathcal{F}, \mathcal{G}$ and $\mathcal{U}$ with the reconstruction and KL losses in energy function~\eqref{Energy_Functional_dataset}.
		\State Choose an augmentation operator $\mathcal{O}$ randomly, and update $\mathcal{U}$ with the binary cross entropy loss~\eqref{Aug_loss}.
		\State Fix decoder $\mathcal{F}$, generate fake foreground/background images $\mathcal{S}_{\Omega_1}$/$\mathcal{S}_{\Omega_2}$, and update discriminator $\mathcal{D}$ with binary classification loss~\eqref{loss_for_d}.
		\State Fix decoder $\mathcal{F}$ and discriminator $\mathcal{D}$, Update $\mathcal{U}$ by~\eqref{CRI_loss}.
		\EndFor \\
		\Return segmentation function $\mathcal{U}$.
	\end{algorithmic} 
\end{algorithm}

\begin{rem}
    The loss function~\eqref{Energy_Functional_dataset} can be derived under the variational inference framework. Specifically, we assume that $I$ has two latent variables: $Z$ is the latent image that satisfies the Gaussian assumption, and $u$ is the segmentation mask for $I$. Maximizing the ELBO obtained by variational inference is equivalent to the minimize loss function we proposed in~\eqref{Energy_Functional_dataset}. From this perspective, our method is interpretable, see section~\ref{app:vi_dataset} in Appendix for the details.
\end{rem}

\section{Experiments and Results}\label{experiments}
In this section, we show results of our segmentation method. All experiments are evaluated in the sRGB space.

\subsection{Implementation Details}
All encoder, decoder maps are parameterized by the U-net~\cite{ronneberger2015u}, which includes an encoder part (the down-sampling) and a decoder part (the up-sampling). For each down-sampling module, it halves the data size and doubles the number of channels. Correspondingly, for each up-sampling module, it doubles the data size and reduces the number of channels by half. U-net includes 4 down-sampling and 4 up-sampling modules. The result of each down-sampling module is transferred to the corresponding up-sampling module through a skip connection. 

For single image based fg/bg segmentation, the regularization parameter $\nu$ is fixed as $1$, the dimension $d$ of the latent space is fixed as $1$, $\mu_1$, $\mu_2$ in latent space are set as $\mu_1=10$, $\mu_2=-10$, and $\Sigma_1$, $\Sigma_2$ are set as 1. For multi-phase segmentation, we set the latent dimension to be the number of segment regions, $\mu_{i}=5e_{i}$, where $e_{i}=(0,\cdots,1\text{(i-th)},\cdots,0)$, $\Sigma_{i}=\mathbf{I}$, and the regularization parameter $\nu$ is fixed as 1. For dataset based segmentation, the latent dimension is fixed as 1, and we set $\mu_1=-3$, $\mu_2=3$, $\Sigma_1=1$, $\Sigma_2=1$. We use the Monte-Carlo method to estimate the reconstruction loss, where the number of sampling is set as 1.

For single image based fg/bg and multi-phase segmentation, we choose $\mathcal{G}^{\sigma}(\cdot) = \mathbf{I}$ to reduce the parameters. In this case, the objective function in~\eqref{Energy_Functional} is reduced to 
\begin{equation}
\begin{aligned}
	E(u,\mathcal{F},\mathcal{G}) =& \frac{1}{2}  E_{\eta}\|\mathcal{F}(\mathcal{G}^{\mu}(I)+ \eta)-I\|^2 + \sum_{x \in \Omega} \nu \|\nabla u_x\|_2 \\
	&+ \frac{1}{2} \left( u_x \|\mathcal{G}^{\mu}(I)_x - \mu_1\|^2 + (1-u_x)\|\mathcal{G}^{\mu}(I)_x - \mu_2\|^2 \right).
\end{aligned}
\end{equation}
The experimental results show that this setting can achieve a satisfactory performance in single image based segmentation tasks. For dataset based image segmentation, we use 4 neural networks in total, i.e., encoder $\mathcal{G}$ decoder $\mathcal{F}$, segmentation network $\mathcal{U}$, and discriminator $\mathcal{D}$. For $\mathcal{F}$, $\mathcal{G}$ and $\mathcal{U}$ we use the U-net, and for $\mathcal{D}$, the network consists of 5 convolutional layers and one fully connected layer. The number of channels of the convolutional layers are set to 32, 64, 128, 256, 512. We use instance normalization~\cite{ulyanov2016instance} to accelerate the training process. In two phase segmentation, we initialize $u$ using a saliency detection method~\cite{peng2016salient}. In multiphase segmentation and dataset based segmentation, we random initialize $u$ and the segmentation function $\mathcal{U}$.

We use ADAM~\cite{kingma2014adam} algorithm to optimize the network parameters. For our single image based model, the learning rate is set to $1e-1$, and for our dataset based segmentation, we set learning rate to $1e-3$, and use mini-batches of size 128. The auto-gradient framework calculates the discretization of the gradient. Experiments with single images are running on a single NVIDIA GeForce GTX 1080TI GPU, and dataset based experiments are running on an 8 $\times$ NVIDIA GeForce GTX 1080TI GPU server.

\begin{figure}
	\centering
	\begin{tabular}{c@{\hspace{0.01\linewidth}}c@{\hspace{0.01\linewidth}}c@{\hspace{0.01\linewidth}}c@{\hspace{0.01\linewidth}}c}
		\includegraphics[width=0.185\linewidth]{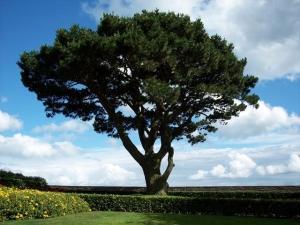} &
		\includegraphics[width=0.185\linewidth]{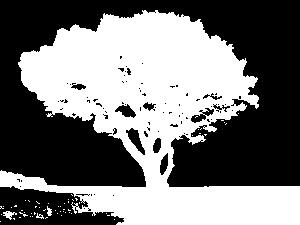} &
		\includegraphics[width=0.185\linewidth]{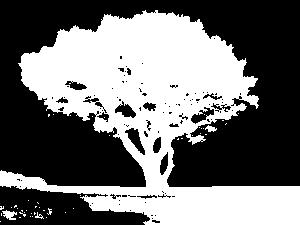} &
		\includegraphics[width=0.185\linewidth]{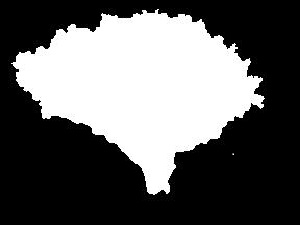} &
		\includegraphics[width=0.185\linewidth]{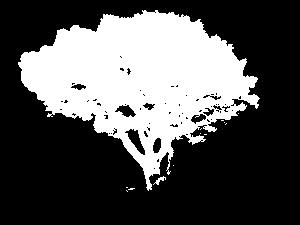} \\
		
		\includegraphics[width=0.185\linewidth]{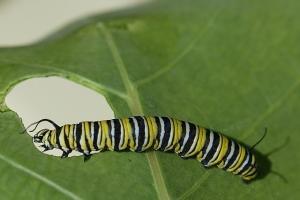} &
		\includegraphics[width=0.185\linewidth]{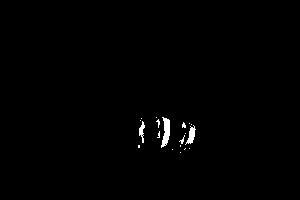} &
		\includegraphics[width=0.185\linewidth]{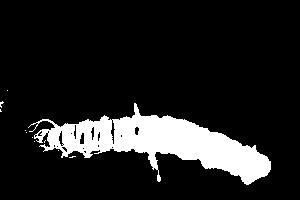} &
		\includegraphics[width=0.185\linewidth]{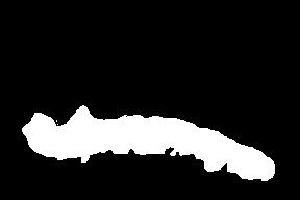} &
		\includegraphics[width=0.185\linewidth]{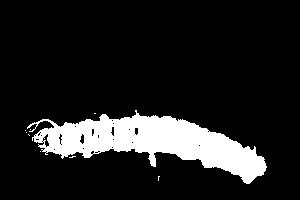} \\
		
		\includegraphics[width=0.185\linewidth]{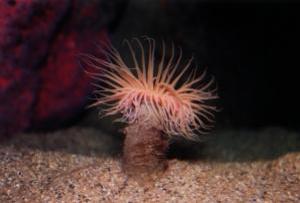} &
		\includegraphics[width=0.185\linewidth]{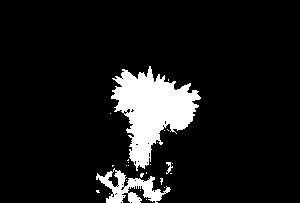} &
		\includegraphics[width=0.185\linewidth]{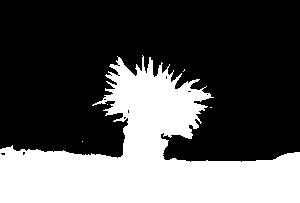} &
		\includegraphics[width=0.185\linewidth]{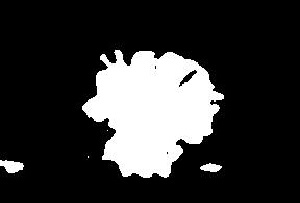} &
		\includegraphics[width=0.185\linewidth]{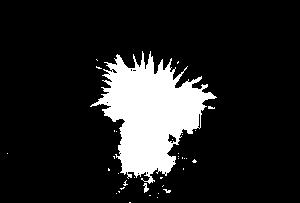} \\
		
		\includegraphics[width=0.185\linewidth]{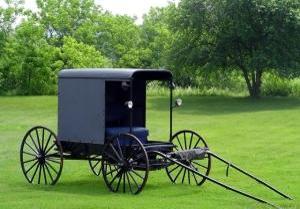} &
		\includegraphics[width=0.185\linewidth]{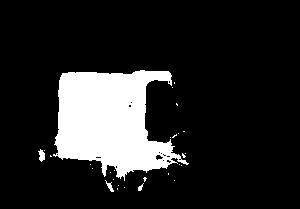} &
		\includegraphics[width=0.185\linewidth]{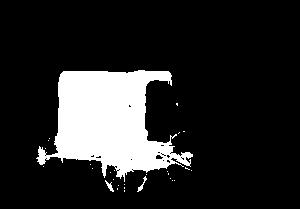} &
		\includegraphics[width=0.185\linewidth]{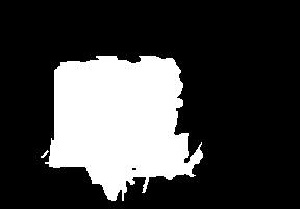} &
		\includegraphics[width=0.185\linewidth]{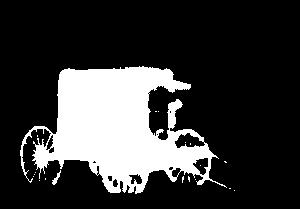} \\
		
		\includegraphics[width=0.185\linewidth]{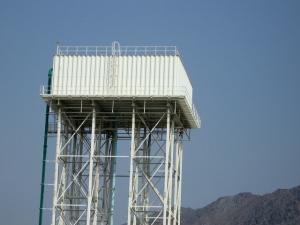} &
		\includegraphics[width=0.185\linewidth]{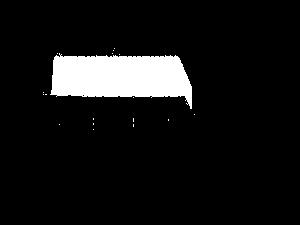} &
		\includegraphics[width=0.185\linewidth]{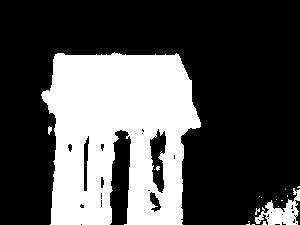} &
		\includegraphics[width=0.185\linewidth]{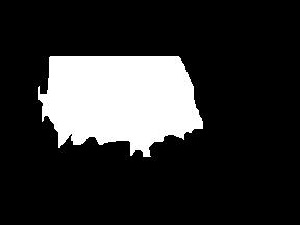} &
		\includegraphics[width=0.185\linewidth]{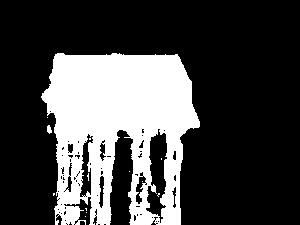} \\
		
		\includegraphics[width=0.185\linewidth]{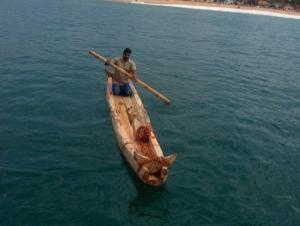} &
		\includegraphics[width=0.185\linewidth]{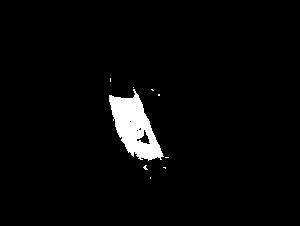} &
		\includegraphics[width=0.185\linewidth]{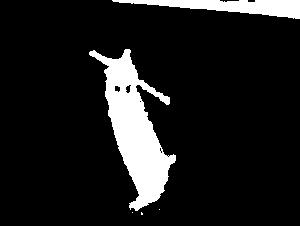} &
		\includegraphics[width=0.185\linewidth]{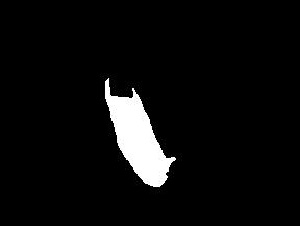} &
		\includegraphics[width=0.185\linewidth]{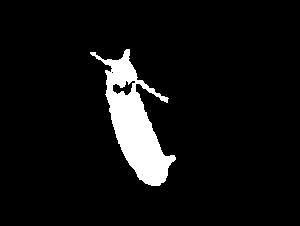} \\
		
		\includegraphics[width=0.185\linewidth]{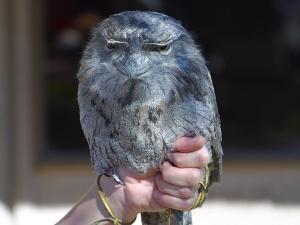} &
		\includegraphics[width=0.185\linewidth]{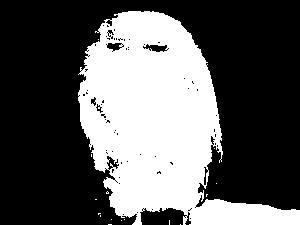} &
		\includegraphics[width=0.185\linewidth]{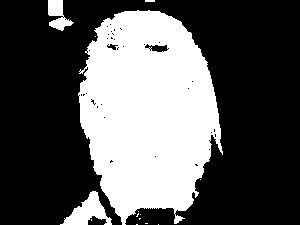} &
		\includegraphics[width=0.185\linewidth]{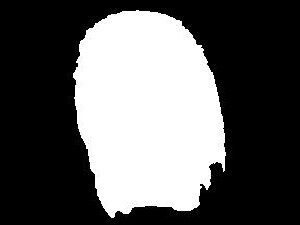} &
		\includegraphics[width=0.185\linewidth]{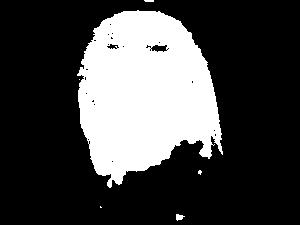} \\
		
		\includegraphics[width=0.185\linewidth]{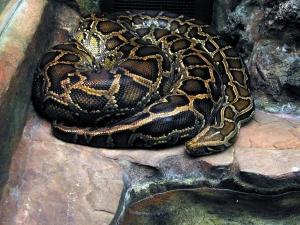} &
		\includegraphics[width=0.185\linewidth]{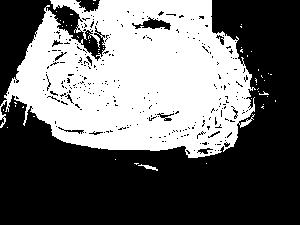} &
		\includegraphics[width=0.185\linewidth]{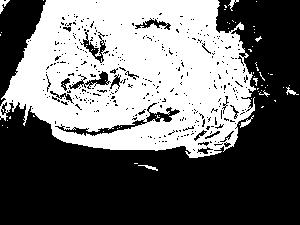} &
		\includegraphics[width=0.185\linewidth]{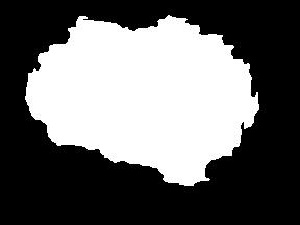} &
		\includegraphics[width=0.185\linewidth]{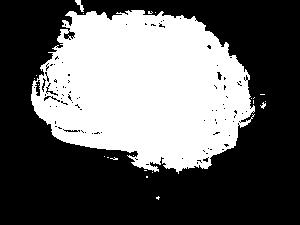} \\
		
		\includegraphics[width=0.185\linewidth]{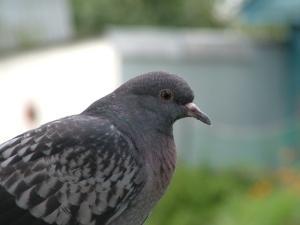} &
		\includegraphics[width=0.185\linewidth]{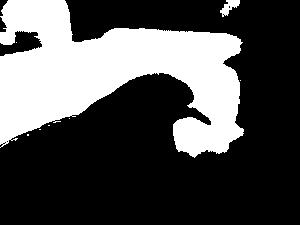} &
		\includegraphics[width=0.185\linewidth]{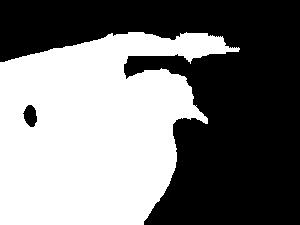} &
		\includegraphics[width=0.185\linewidth]{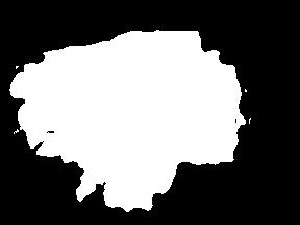} &
		\includegraphics[width=0.185\linewidth]{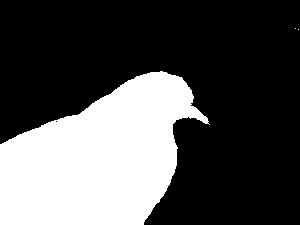} \\
		
		{\small{(a)} Inputs} & {\small{(b)} CV \cite{chan2001active}} 
		& {\small{(c)} SA \cite{cremers2007review}} & 
		{\small{(d)} D-DIP \cite{gandelsman2018double}} &{\small{(f)} Ours}
	\end{tabular}
	\caption{Visual results of our fg/bg segmentation on Weizmann dataset.}
	\label{fig: seg_results}
\end{figure}

\begin{table}
	\centering
	\caption{Segment coverage test results on Weizmann database.}
	\label{Weizmann_1obj_result}
	\begin{tabular}{cccccc}
		\hline
		& D-DIP~\cite{gandelsman2018double}          & N-Cuts~\cite{shi2000normalized}         & CV~\cite{chan2001active}             & SA~\cite{cremers2007review}             & Ours           \\ \hline
		F-measure & 0.83 & 0.70 & 0.79 & 0.85 & \textbf{0.87} \\ \hline
		mIoU      & 0.73 & 0.56 & 0.68 & 0.77 & \textbf{0.79} \\ \hline
	\end{tabular}
\end{table}

\subsection{Fg/bg Segmentation}
We evaluate our fg/bg segmentation method on the Weizmann dataset~\cite{AlpertGBB07}, which contains 100 color images with fg/bg segmentation results manually annotated by people. The Double-DIP model~\cite{gandelsman2018double}, the Chan-Vese model~\cite{chan2001active}, the statistical model~\cite{cremers2007review} and the Normal Cut method~\cite{shi2000normalized} are chosen for comparison. Following the same settings in the Double-DIP model~\cite{gandelsman2018double}, we use the result of a saliency detection method~\cite{peng2016salient} to make an initialization of the segmentation mask $u$ for Ours model, the Double-DIP model~\cite{gandelsman2018double}, the Chan-Vese model~\cite{chan2001active} and the statistical model~\cite{cremers2007review}. For the Normal Cut method~\cite{shi2000normalized}, we evaluate different numbers of regions range from 2 to 10 and choose the best for the whole dataset. We further apply guided filtering~\cite{he2010guided} on the segmentation result to obtain a refined result. 

For quantitative analysis, we evaluate segmentation results by assessing their consistency with the ground truth segmentation. F-measure and mIoU results of each method are evaluated here. Denote $\mathrm{TP}$, $\mathrm{FP}$, $\mathrm{TN}$ and $\mathrm{FN}$ the true positive, false positive, true negative, and false negative values of a particular segmentation, then
\begin{equation}
	\mathrm{Recall}=\frac{\mathrm{TP}}{\mathrm{TP}+\mathrm{FN}} \quad \mathrm{Precision}=\frac{\mathrm{TP}}{\mathrm{TP}+\mathrm{FP}},
\end{equation}
F-measure is the combination of Recall ($\mathrm{R}$) and Precision ($\mathrm{P}$)
\begin{equation}
	\text{F-measure}=\frac{2\mathrm{R}\mathrm{P}}{\mathrm{R}+\mathrm{P}}.
\end{equation}
Whereas mean intersection over union (mIoU) is the ratio between the area of the intersection between the inferred segmentation and the ground truth over the area of their union
\begin{equation}
	\mathrm{mIoU} = \frac{\mathrm{TP}}{\mathrm{FN}+\mathrm{FP}+\mathrm{TP}}.
\end{equation}
See Table~\ref{Weizmann_1obj_result} for the results. From the table, we see that our method gives the highest scores for both F-measure and mIoU. The visual results are given in Figure~\ref{fig: seg_results}, where we find our method achieves more accurate segmentation results than other methods.

\begin{figure}
	\centering
	\begin{tabular}{c@{\hspace{0.01\linewidth}}c@{\hspace{0.01\linewidth}}c@{\hspace{0.01\linewidth}}c}
		\includegraphics[width=0.23\linewidth]{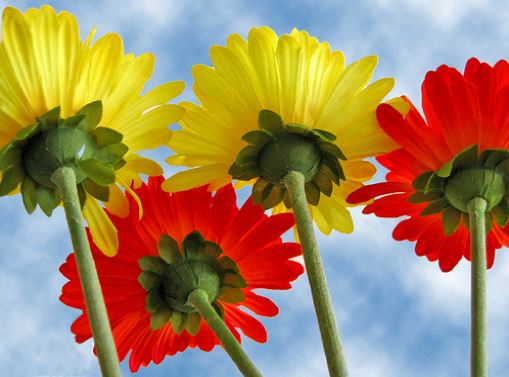} &
		\includegraphics[width=0.23\linewidth]{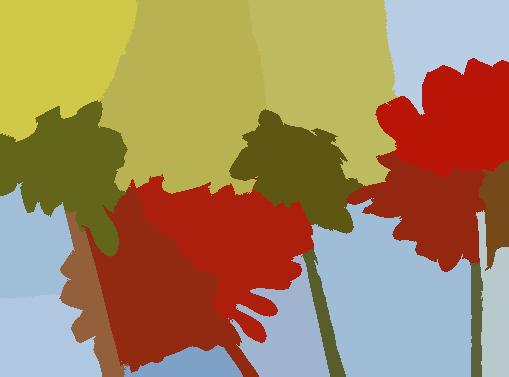} &
		\includegraphics[width=0.23\linewidth]{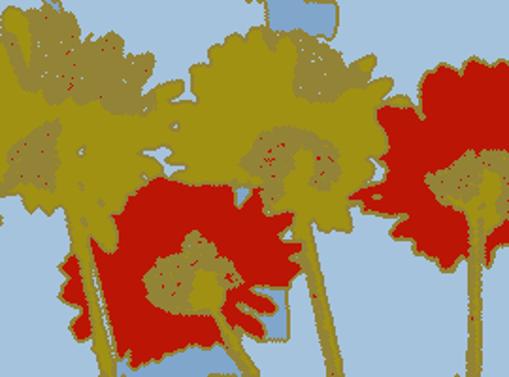} &
		\includegraphics[width=0.23\linewidth]{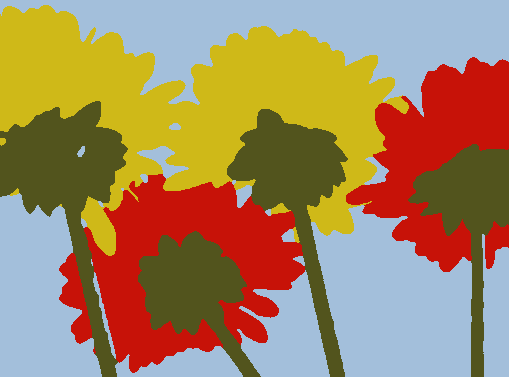} \\
		
		\includegraphics[width=0.23\linewidth]{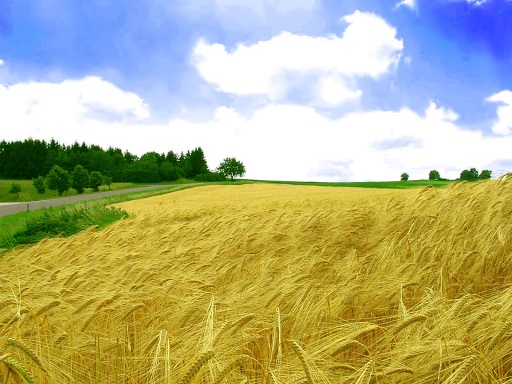} &
		\includegraphics[width=0.23\linewidth]{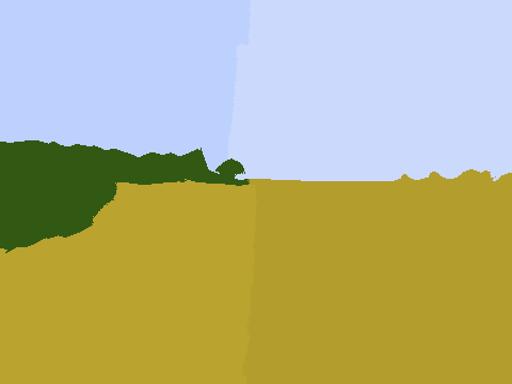} &
		\includegraphics[width=0.23\linewidth]{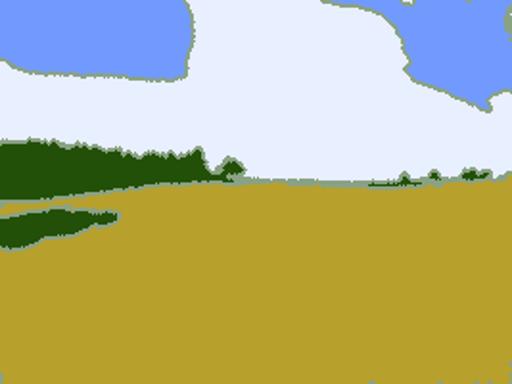} &
		\includegraphics[width=0.23\linewidth]{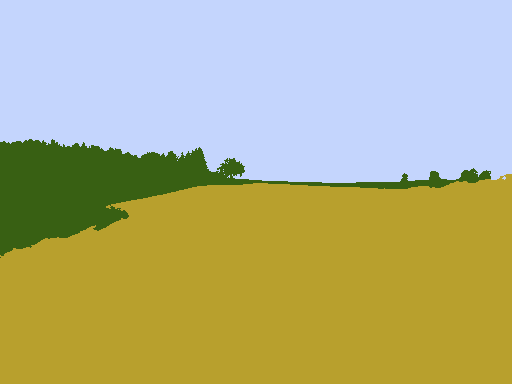} \\
		
		\includegraphics[width=0.23\linewidth]{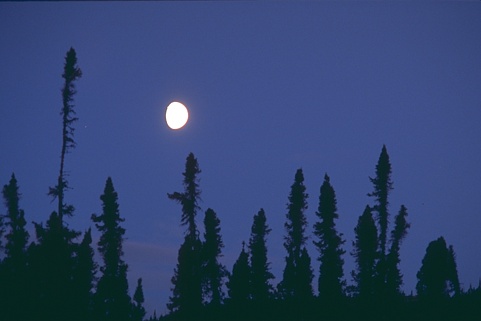} &
		\includegraphics[width=0.23\linewidth]{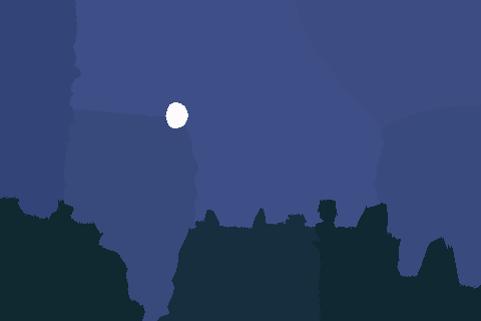} &
		\includegraphics[width=0.23\linewidth]{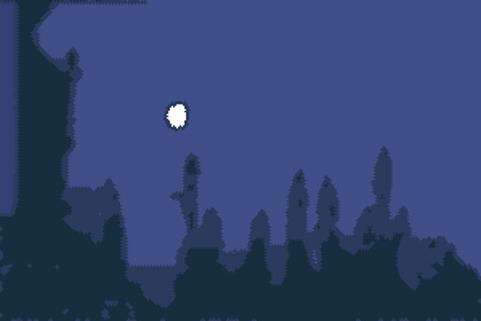} &
		\includegraphics[width=0.23\linewidth]{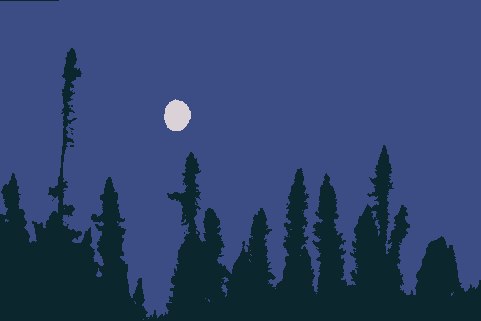} \\
		
		\includegraphics[width=0.23\linewidth]{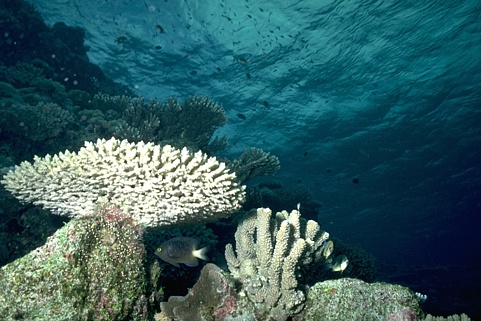} &
		\includegraphics[width=0.23\linewidth]{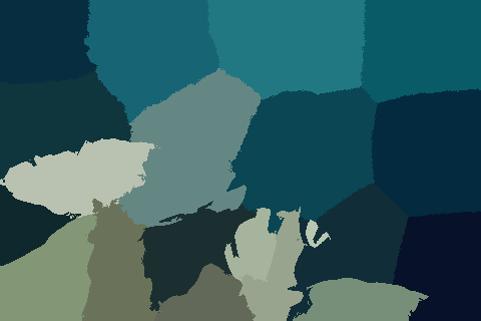} &
		\includegraphics[width=0.23\linewidth]{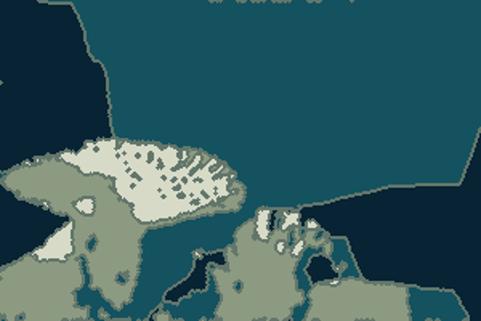} &
		\includegraphics[width=0.23\linewidth]{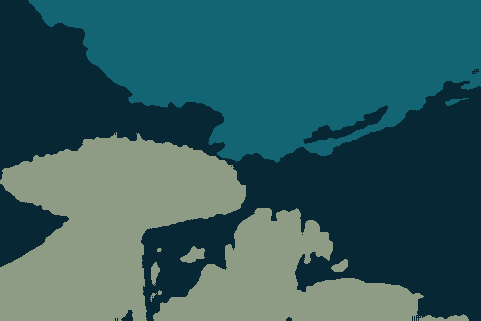} \\
		
		{\small{(a)} Inputs} & {\small{(b)} N-Cut~\cite{shi2000normalized}} 
		& {\small{(c)} CV~\cite{vese2002multiphase}} & {\small{(d)} Ours} 
	\end{tabular}
	\caption{Multi-phase image segmentation.}
	\label{fig: multiseg_results}
\end{figure}

\subsection{Multi-phase Segmentation}
Four images are chosen for testing the performance of the Normal Cut model~\cite{shi2000normalized}, the multi-phase version Chan-Vese model~\cite{vese2002multiphase}, and our proposed method. The results are shown in Figure~\ref{fig: multiseg_results}. For Normal Cut model~\cite{shi2000normalized}, we set the number of regions ranging from 3 to 20 and choose the best one. For the multi-phase Chan-Vese model~\cite{vese2002multiphase}, we use two level-set functions to represent the regions. For our model, we set the region's number to 4, 3, 3, 3 for the test images in Figure~\ref{fig: multiseg_results}. We find our model achieves a satisfactory result for all images while Normal Cut and Chan-Vese model are less accurate.

\begin{figure}
	\centering
	\begin{tabular}{c@{\hspace{0.005\linewidth}}c@{\hspace{0.005\linewidth}}c@{\hspace{0.005\linewidth}}c@{\hspace{0.005\linewidth}}c@{\hspace{0.005\linewidth}}c@{\hspace{0.005\linewidth}}c@{\hspace{0.005\linewidth}}c@{\hspace{0.005\linewidth}}c@{\hspace{0.005\linewidth}}c}
		\includegraphics[width=0.092\linewidth]{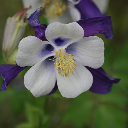}     &
		\includegraphics[width=0.092\linewidth]{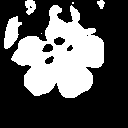} &
		\includegraphics[width=0.092\linewidth]{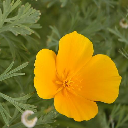}     &
		\includegraphics[width=0.092\linewidth]{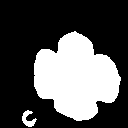} &
		\includegraphics[width=0.092\linewidth]{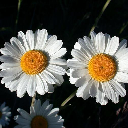}     &
		\includegraphics[width=0.092\linewidth]{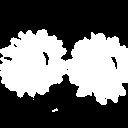} &
		\includegraphics[width=0.092\linewidth]{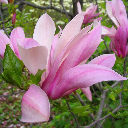}     &
		\includegraphics[width=0.092\linewidth]{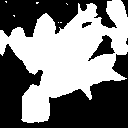} &
		\includegraphics[width=0.092\linewidth]{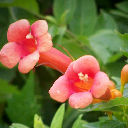}     &
		\includegraphics[width=0.092\linewidth]{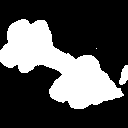} \\
		
		\includegraphics[width=0.092\linewidth]{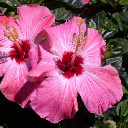}     &
		\includegraphics[width=0.092\linewidth]{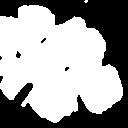} &
		\includegraphics[width=0.092\linewidth]{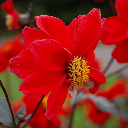}     &
		\includegraphics[width=0.092\linewidth]{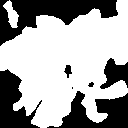} &
		\includegraphics[width=0.092\linewidth]{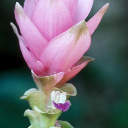}     &
		\includegraphics[width=0.092\linewidth]{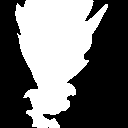} &
		\includegraphics[width=0.092\linewidth]{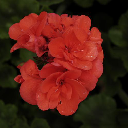}     &
		\includegraphics[width=0.092\linewidth]{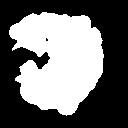} &
		\includegraphics[width=0.092\linewidth]{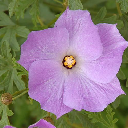}    &
		\includegraphics[width=0.092\linewidth]{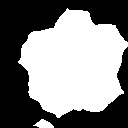}\\
		
		\includegraphics[width=0.092\linewidth]{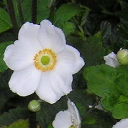}     &
		\includegraphics[width=0.092\linewidth]{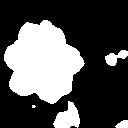} &
		\includegraphics[width=0.092\linewidth]{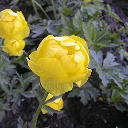}     &
		\includegraphics[width=0.092\linewidth]{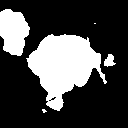} &
		\includegraphics[width=0.092\linewidth]{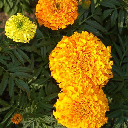}     &
		\includegraphics[width=0.092\linewidth]{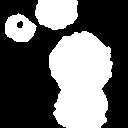} &
		\includegraphics[width=0.092\linewidth]{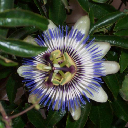}     &
		\includegraphics[width=0.092\linewidth]{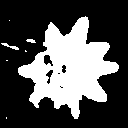} &
		\includegraphics[width=0.092\linewidth]{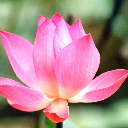}     &
		\includegraphics[width=0.092\linewidth]{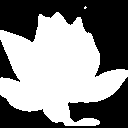} \\
		
		\includegraphics[width=0.092\linewidth]{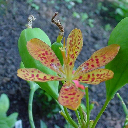}     &
		\includegraphics[width=0.092\linewidth]{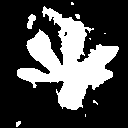} &
		\includegraphics[width=0.092\linewidth]{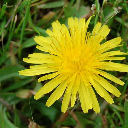}     &
		\includegraphics[width=0.092\linewidth]{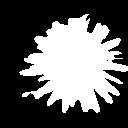} &
		\includegraphics[width=0.092\linewidth]{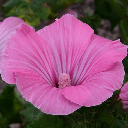}     &
		\includegraphics[width=0.092\linewidth]{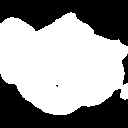} &
		\includegraphics[width=0.092\linewidth]{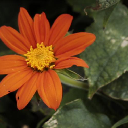}     &
		\includegraphics[width=0.092\linewidth]{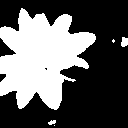} &
		\includegraphics[width=0.092\linewidth]{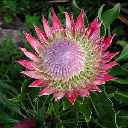}     &
		\includegraphics[width=0.092\linewidth]{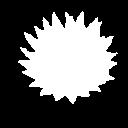} \\
	\end{tabular}
	\caption{Visual results of our dataset segmentation method on the Flower test dataset.}
	\label{flower_visual}
\end{figure}

\subsection{Dataset Based Segmentation}
We evaluate our dataset based segmentation method on the Flower dataset~\cite{nilsback2007delving,nilsback2008automated}. The dataset is provided with a set of masks obtained via an automated method built specifically for flowers~\cite{nilsback2007delving}. We split into sets of 6149 training images, 1020 validation, and 1020 test images, same as in ReDO~\cite{chen2019unsupervised}, and use the provided masks as ground truth for evaluation purposes only. All images have been resized and then cropped to $128 \times 128$. To evaluate our method, we use the commonly used metrics as ReDO~\cite{chen2019unsupervised}. The pixel classification accuracy (Acc) measures the proportion of pixels that have been assigned to the correct region. The mean intersection over union (mIoU) is the ratio between the intersection area between the inferred mask and the ground truth over the area of their union. In both cases, higher is better. We train our model on the training set without the ground truth mask and save the models with the highest mIoU score on the validation set within 50 epochs. For quantitative analysis, we compare our method with the Chan-Vese method~\cite{chan2001active}, and the ReDO method~\cite{chen2019unsupervised}, see Table~\ref{Flower_result} for the results. From the table, we find our method is greatly improved compared to the Chan-Vese method and achieves better results than the recently proposed method ReDO on both training and testing sets. Moreover, after the training process, our method inference much faster than the original Chan-Vese method. To process one $128 \times 128$ images, Chan-Vese takes 1.95s, while ours is 0.02s. See Figure~\ref{flower_visual} for the visual results of our method.

\noindent{\bf Ablation Study.} We use ablation experiments to analyze the effect of the regularization used in our method. For each experiment we train a individual model and evaluate its performance on the Flower dataset using the same settings as in the main experiment. See Table~\ref{ablation_study} for the results, where "AuI" represents "Augmentation Invariant", and "CRI" represents "Conservation of Region Information". From the table, we find both constraints have relatively little influence on the final results. This is because the purpose of these constraints is to prevent model collapse. For example, the network $\mathcal{U}$ output empty foreground or background for all inputs, then model will degenerate to the traditional VAE model, where the "Conservation of Region Information" constraint can prevent this situation.

\begin{table}
	\centering
	\caption{Segmentation results on on Flower dataset}
	\label{Flower_result}
	\begin{tabular}{ccccc}
		\hline
		& Train Acc & Train mIoU & Test Acc & Test mIoU \\ \hline
		CV~\cite{chan2001active} & 0.569     & 0.366      & 0.567    & 0.357     \\ \hline
		ReDO~\cite{chen2019unsupervised}      & 0.886     & 0.789      & 0.879    & 0.764     \\ \hline
		Ours      & \textbf{0.901}     & \textbf{0.796}      & \textbf{0.891}    & \textbf{0.778}     \\ \hline
	\end{tabular}
\end{table}

\begin{table}[]
	\centering
	\caption{Ablation study on the Flower dataset.}
	\label{ablation_study}
	\begin{tabular}{cccccc}
		\hline
		AuI          &    CRI       & Train Acc & Train mIoU & Test Acc & Test mIoU \\ \hline
		             &              & 0.887          & 0.778          & 0.879          & 0.761          \\ \hline
		$\checkmark$ &              & 0.889          & 0.782          & 0.882          & 0.767          \\ \hline
		             & $\checkmark$ & 0.894          & 0.785          & 0.886          & 0.766          \\ \hline
		$\checkmark$ & $\checkmark$ & \textbf{0.901} & \textbf{0.796} & \textbf{0.891} & \textbf{0.778} \\ \hline
	\end{tabular}
\end{table}

\begin{figure}
	\centering
	\begin{tabular}{c@{\hspace{0.01\linewidth}}c@{\hspace{0.01\linewidth}}c@{\hspace{0.01\linewidth}}c@{\hspace{0.01\linewidth}}c@{\hspace{0.01\linewidth}}c}
		\includegraphics[width=0.15\linewidth]{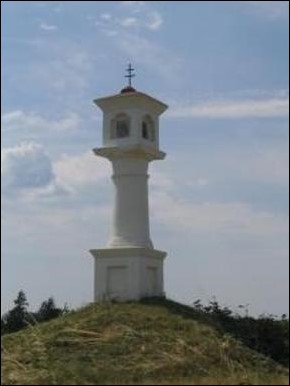} &
		\includegraphics[width=0.15\linewidth]{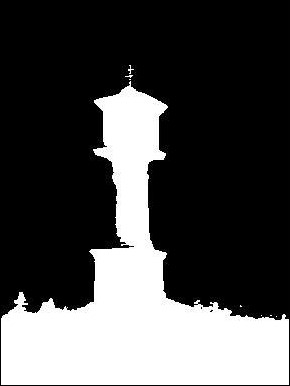} &
		\includegraphics[width=0.15\linewidth]{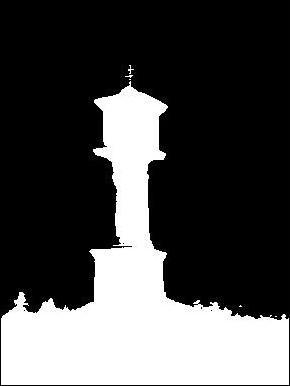} &
		\includegraphics[width=0.15\linewidth]{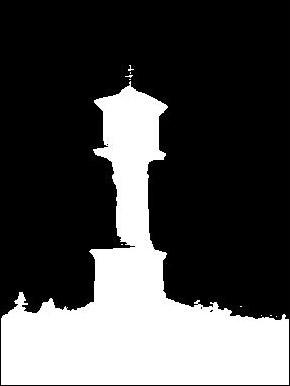} &
		\includegraphics[width=0.15\linewidth]{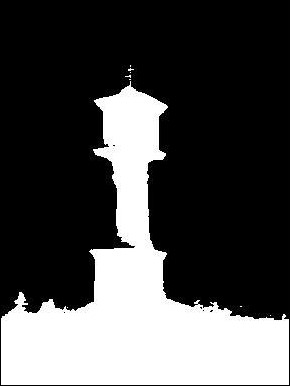} &
		\includegraphics[width=0.15\linewidth]{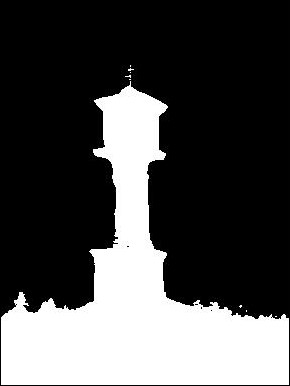} \\
		\includegraphics[width=0.15\linewidth]{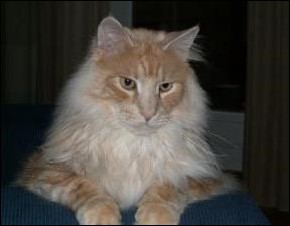} &
		\includegraphics[width=0.15\linewidth]{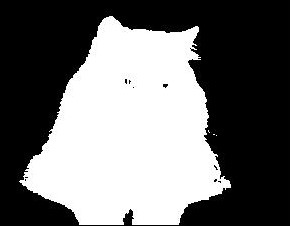} &
		\includegraphics[width=0.15\linewidth]{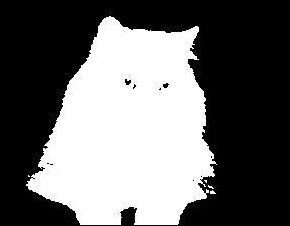} &
		\includegraphics[width=0.15\linewidth]{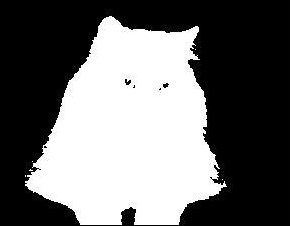} &
		\includegraphics[width=0.15\linewidth]{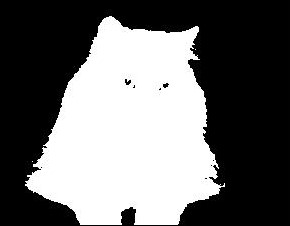} &
		\includegraphics[width=0.15\linewidth]{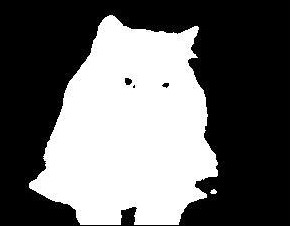} \\
		{\small{(a)}} & {\small{(b)}} 
		& {\small{(c)}} & 
		{\small{(d)}} & {\small{(e)}} &{\small{(f)}}
	\end{tabular}
	\caption{ Segmentation results with different architectures of decoder/encoder maps. (a) Input images. (b)-(e) Network \cite{ulyanov2018deep} with depth $2$-$5$. (f) U-net \cite{ronneberger2015u}.}
	\label{fig: net_robustness}
\end{figure}

\begin{figure}
	\centering
	\begin{tabular}{c@{\hspace{0.004\linewidth}}c@{\hspace{0.004\linewidth}}c@{\hspace{0.004\linewidth}}c@{\hspace{0.004\linewidth}}c@{\hspace{0.004\linewidth}}c@{\hspace{0.004\linewidth}}c@{\hspace{0.004\linewidth}}c@{\hspace{0.004\linewidth}}c@{\hspace{0.004\linewidth}}c}
		\includegraphics[width=0.092\linewidth]{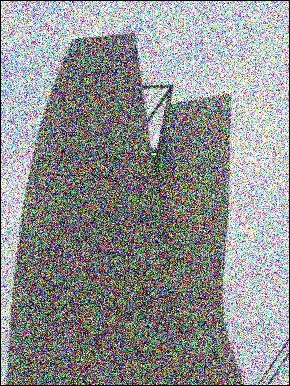} &
		\includegraphics[width=0.092\linewidth]{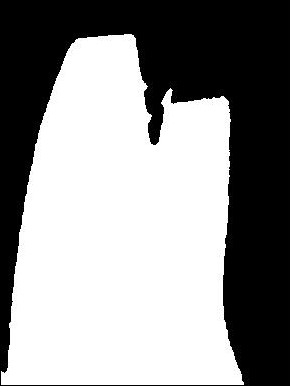} &
		\includegraphics[width=0.092\linewidth]{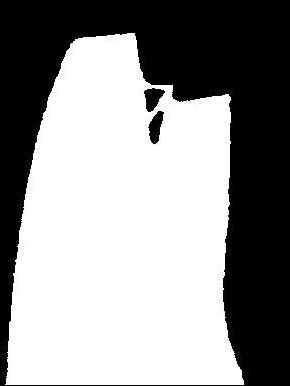} &
		\includegraphics[width=0.092\linewidth]{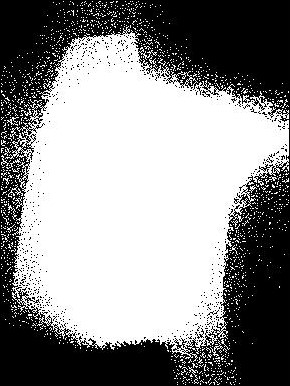} &
		\includegraphics[width=0.092\linewidth]{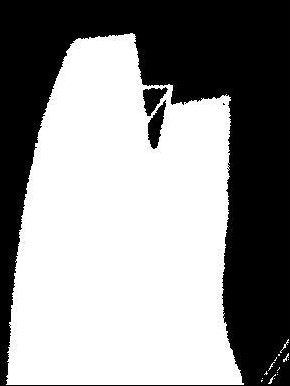} &
		\includegraphics[width=0.092\linewidth]{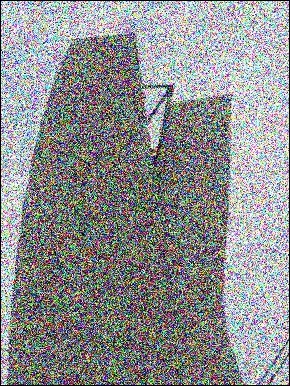} &
		\includegraphics[width=0.092\linewidth]{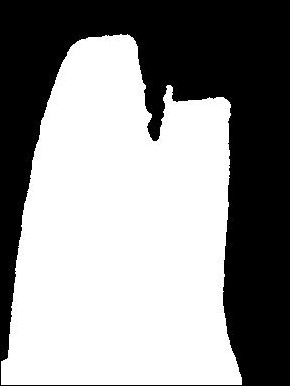} &
		\includegraphics[width=0.092\linewidth]{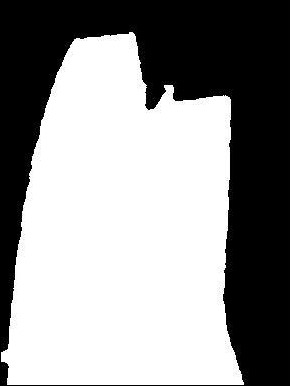} &
		\includegraphics[width=0.092\linewidth]{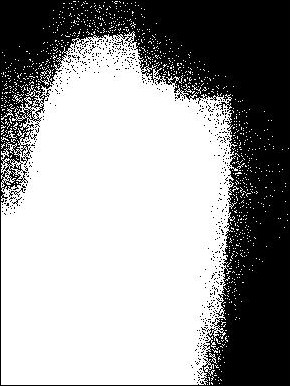} &
		\includegraphics[width=0.092\linewidth]{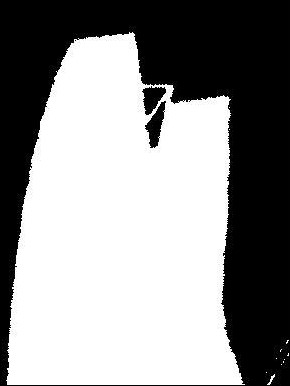} \\
		
		\scriptsize{$\sigma=100$} & \scriptsize{CV~\cite{chan2001active}} & \scriptsize{SA~\cite{cremers2007review}} & \scriptsize{D-DIP~\cite{gandelsman2018double}} & \scriptsize{Ours} & \scriptsize{$\sigma=120$} & \scriptsize{CV~\cite{chan2001active}} & \scriptsize{SA~\cite{cremers2007review}} & \scriptsize{D-DIP~\cite{gandelsman2018double}} & \scriptsize{Ours} \\
		
		\includegraphics[width=0.092\linewidth]{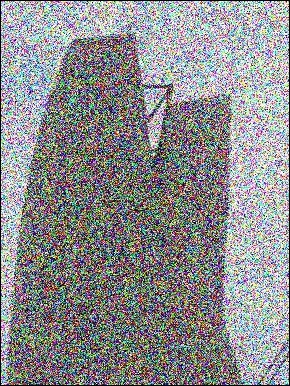} &
		\includegraphics[width=0.092\linewidth]{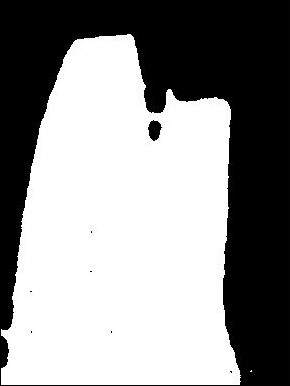} &
		\includegraphics[width=0.092\linewidth]{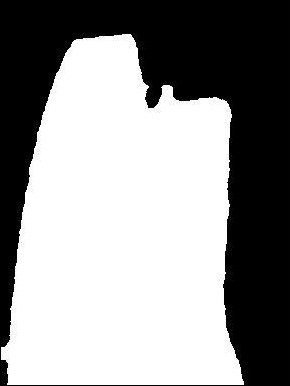} &
		\includegraphics[width=0.092\linewidth]{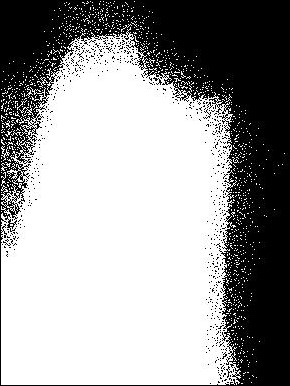} &
		\includegraphics[width=0.092\linewidth]{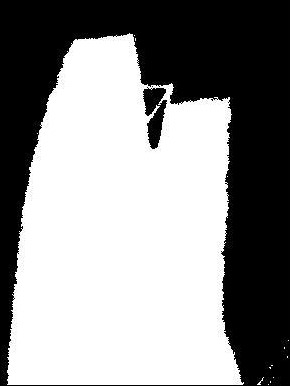}  &
		\includegraphics[width=0.092\linewidth]{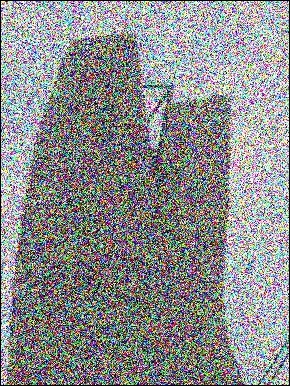} &
		\includegraphics[width=0.092\linewidth]{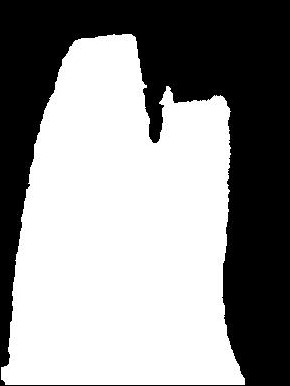} &
		\includegraphics[width=0.092\linewidth]{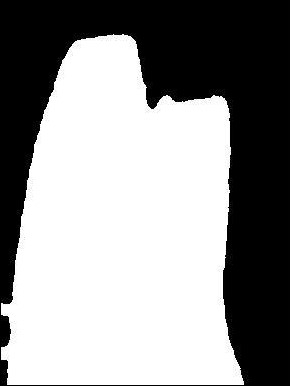} &
		\includegraphics[width=0.092\linewidth]{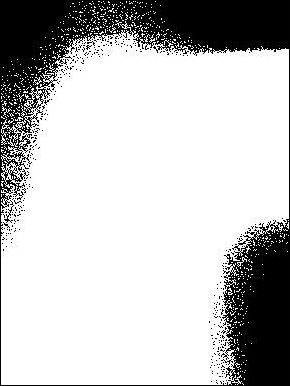} &
		\includegraphics[width=0.092\linewidth]{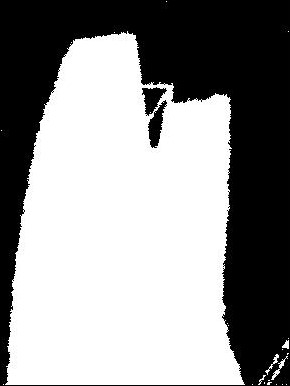}  \\
		
		\scriptsize{$\sigma=140$} & \scriptsize{CV~\cite{chan2001active}} & \scriptsize{SA~\cite{cremers2007review}} & \scriptsize{D-DIP~\cite{gandelsman2018double}} & \scriptsize{Ours} & \scriptsize{$\sigma=160$} & \scriptsize{CV~\cite{chan2001active}} & \scriptsize{SA~\cite{cremers2007review}} & \scriptsize{D-DIP~\cite{gandelsman2018double}} & \scriptsize{Ours} 
	\end{tabular}
	\caption{Noisy image segmentation.}
	\label{fig:gau_noise_seg}
\end{figure}

\section{Discussion}\label{discussion}
In this section, we evaluate the performance of our segmentation method in the following six perspectives:

\noindent{\bf \underline{Different architectures.}} Two networks~\cite{ulyanov2018deep} and~\cite{ronneberger2015u} as the decoder and encoder maps i.e., $\mathcal{F}$, $\mathcal{G}$, are tested, the results are shown in Figure~\ref{fig: net_robustness}. For network~\cite{ulyanov2018deep}, we test for different network depths from 2 to 5 while for network~\cite{ronneberger2015u} we use the standard network structure of the original model. It is shown that the segmentation results are stable as the depth or architecture of the network varies.

\begin{figure}
	\centering
	\begin{tabular}{c@{\hspace{0.005\linewidth}}c@{\hspace{0.005\linewidth}}c@{\hspace{0.005\linewidth}}c@{\hspace{0.005\linewidth}}c@{\hspace{0.005\linewidth}}c@{\hspace{0.005\linewidth}}c@{\hspace{0.005\linewidth}}c}
		\includegraphics[width=0.115\linewidth,height=0.09\linewidth]{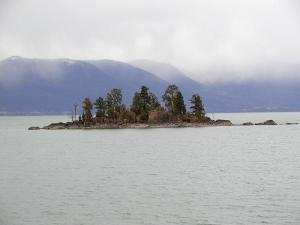} & 
		\includegraphics[width=0.115\linewidth,height=0.09\linewidth]{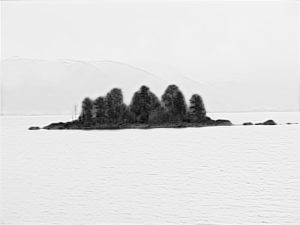} & 
		\includegraphics[width=0.115\linewidth,height=0.09\linewidth]{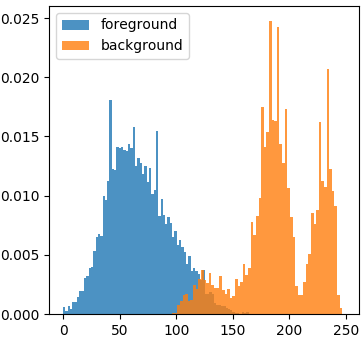} & 
		\includegraphics[width=0.115\linewidth,height=0.09\linewidth]{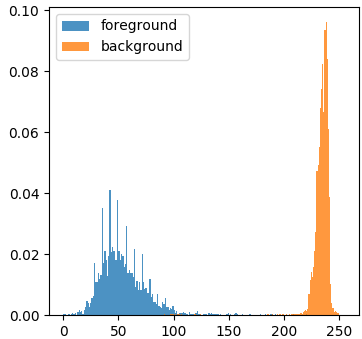} &
		\includegraphics[width=0.115\linewidth,height=0.09\linewidth]{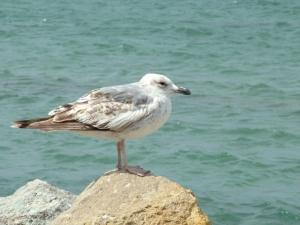} & 
		\includegraphics[width=0.115\linewidth,height=0.09\linewidth]{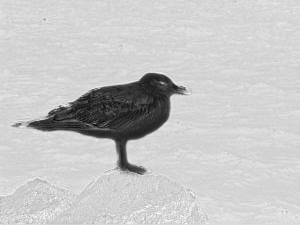} & 
		\includegraphics[width=0.115\linewidth,height=0.09\linewidth]{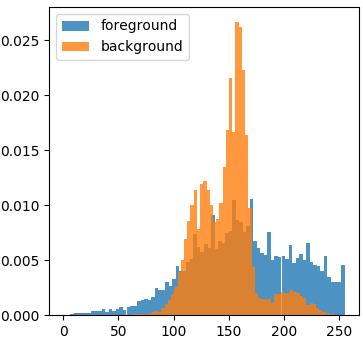} & 
		\includegraphics[width=0.115\linewidth,height=0.09\linewidth]{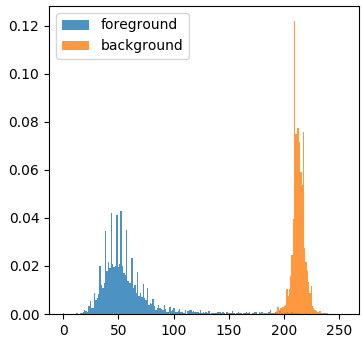} \\
		
		\includegraphics[width=0.115\linewidth,height=0.09\linewidth]{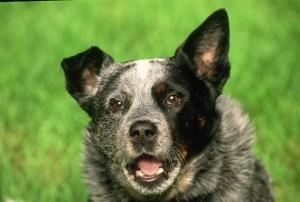} & 
		\includegraphics[width=0.115\linewidth,height=0.09\linewidth]{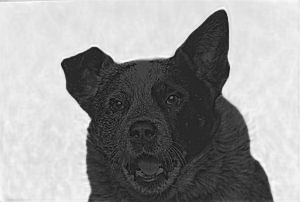} & 
		\includegraphics[width=0.115\linewidth,height=0.09\linewidth]{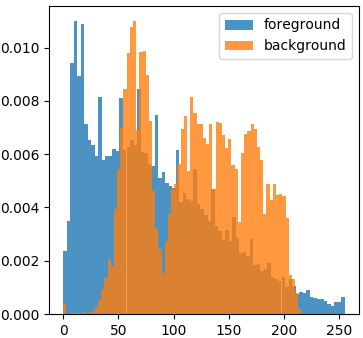} & 
		\includegraphics[width=0.115\linewidth,height=0.09\linewidth]{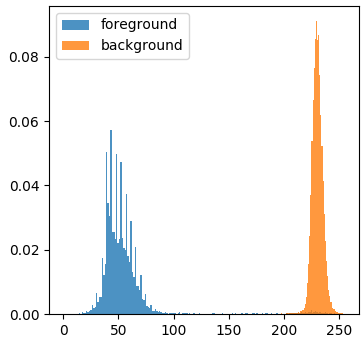} &
		\includegraphics[width=0.115\linewidth,height=0.09\linewidth]{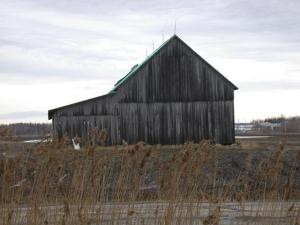} & 
		\includegraphics[width=0.115\linewidth,height=0.09\linewidth]{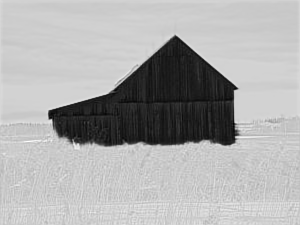} & 
		\includegraphics[width=0.115\linewidth,height=0.09\linewidth]{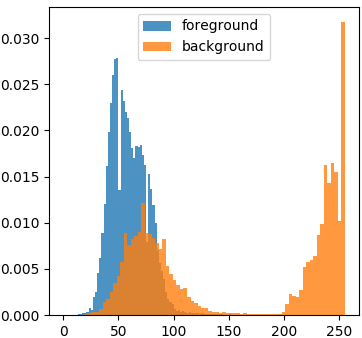} & 
		\includegraphics[width=0.115\linewidth,height=0.09\linewidth]{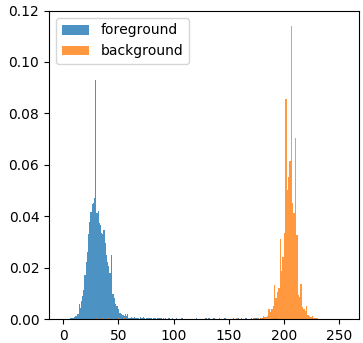} \\
		
		\small{(a1)} & \small{(b1)} & \small{(c1)} & \small{(d1)} & \small{(a2)} & \small{(b2)} & \small{(c2)} & \small{(d2)}
	\end{tabular}
	\caption{Visualization of latent representations. (a1), (a2) input images. (b1), (b2) latent representations. (c1), (c2), (d1), (d2) foreground/background distributions of input images and latent representations.}
	\label{fig:latent}
\end{figure}

\begin{figure}
	\centering
	\begin{tabular}{c@{\hspace{0.005\linewidth}}c@{\hspace{0.005\linewidth}}c@{\hspace{0.005\linewidth}}c@{\hspace{0.005\linewidth}}c@{\hspace{0.005\linewidth}}c@{\hspace{0.005\linewidth}}c@{\hspace{0.005\linewidth}}c}
		\includegraphics[width=0.115\linewidth]{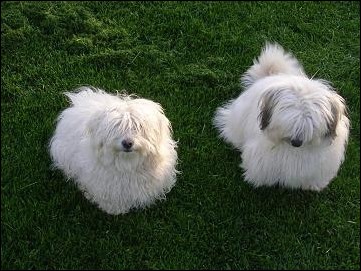} & 
		\includegraphics[width=0.115\linewidth]{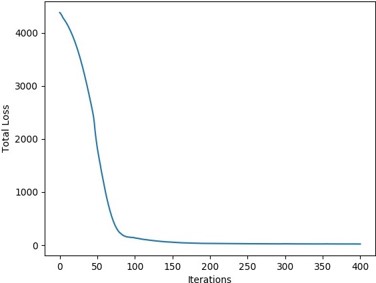} &
		\includegraphics[width=0.115\linewidth]{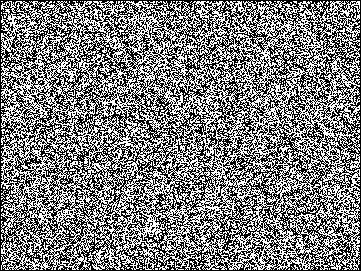} &
		\includegraphics[width=0.115\linewidth]{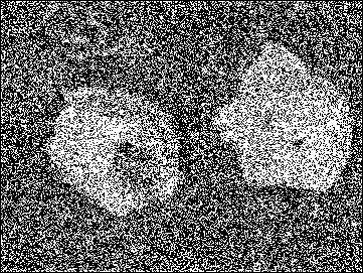} &
		\includegraphics[width=0.115\linewidth]{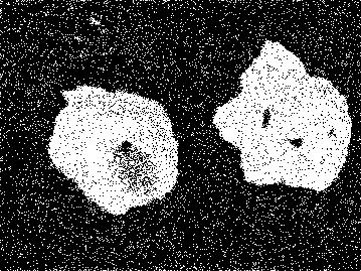} &
		\includegraphics[width=0.115\linewidth]{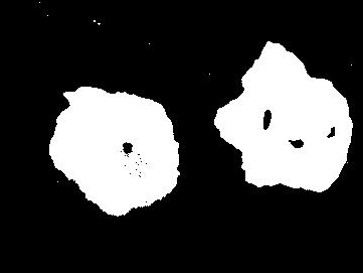} &
		\includegraphics[width=0.115\linewidth]{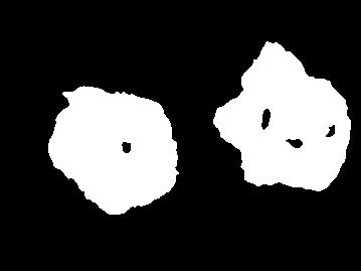} &
		\includegraphics[width=0.115\linewidth]{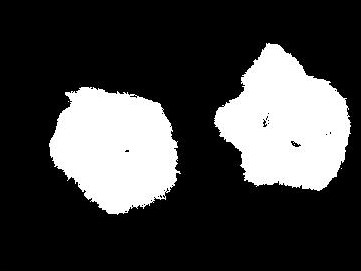} \\
		
		\includegraphics[width=0.115\linewidth]{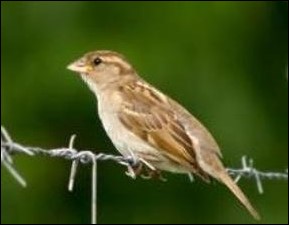} & 
		\includegraphics[width=0.115\linewidth]{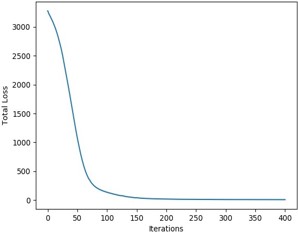} &
		\includegraphics[width=0.115\linewidth]{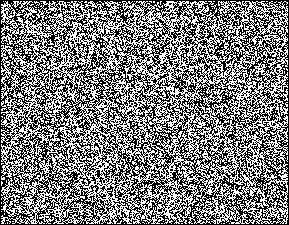} &
		\includegraphics[width=0.115\linewidth]{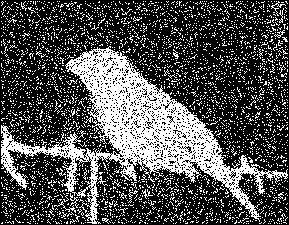} &
		\includegraphics[width=0.115\linewidth]{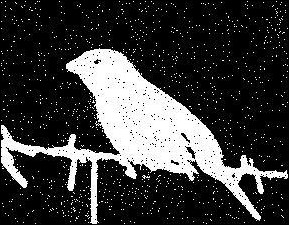} &
		\includegraphics[width=0.115\linewidth]{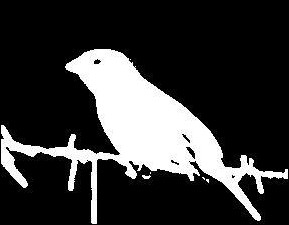} &
		\includegraphics[width=0.115\linewidth]{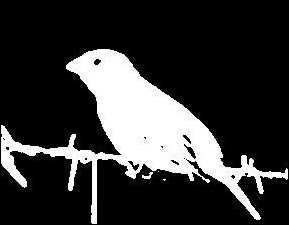} &
		\includegraphics[width=0.115\linewidth]{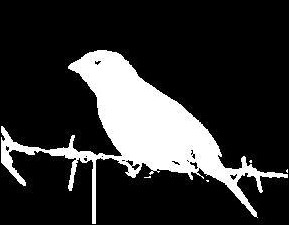} \\
		
		\small{(a)} & \small{(b)} & \small{(c)} & \small{(d)} & \small{(e)} & \small{(f)} & \small{(g)} & \small{(h)}
	\end{tabular}
	\caption{Segmentation process. (a) input images. (b) loss curves. (c)-(g) results of 0 iteration, 10 iteration, 20 iteration, 40 interaction and 100 iteration. (h) final results.}
	\label{fig:loss_curve}
\end{figure}

\noindent{\bf\underline{Latent space visualization.}} As the latent dimension is set as 1 in fg/bg segmentation, we can visualize the latent space representations of the Weizmann dataset images in Figure~\ref{fig:latent}. In latent space, the contrast between fg and bg is magnified by the encoding map $\mathcal{G}$. This magnification ignores irrelevant contents in intensity space and simplifies the representations, crucial for accurate segmentation. Also, we see the fg/bg distributions in image space are not Gaussian while fg/bg of the latent representations are more like Gaussian distributions.

\noindent{\bf\underline{Noisy image segmentation.}} Let input images (range from 0 to 255) be corrupted by the additive Gaussian noise with mean zero and $\sigma=$100, 120, 140 and 160. We compare our model with Chan-Vese model~\cite{chan2001active}, statistical model~\cite{cremers2007review} and Double-DIP method~\cite{gandelsman2018double}. Due to noise, all the compared methods failed while our method can still capture enough details for all noise levels as shown in Figure~\ref{fig:gau_noise_seg}.

\noindent{\bf\underline{The dynamics of the proposed method.}} The segmentation process of our approach is shown in Figure~\ref{fig:loss_curve}. The total loss has the energy dissipation property, and the result is improving as the iteration increases, which verifies that our algorithm is stable and convergent numerically.

\begin{figure}
    \centering
    \begin{tabular}{cc}
        \includegraphics[width=0.45\linewidth]{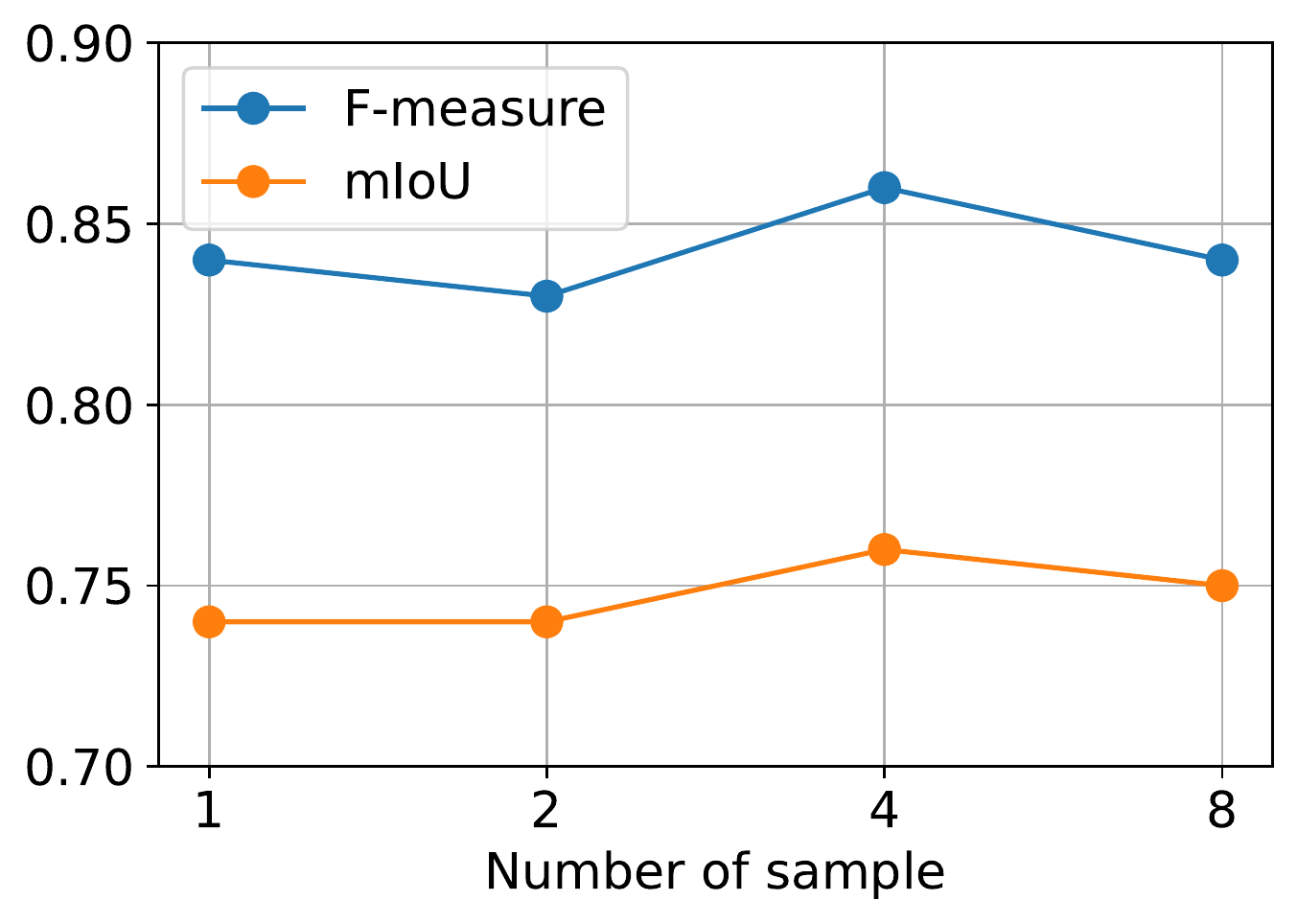} &
        \includegraphics[width=0.45\linewidth]{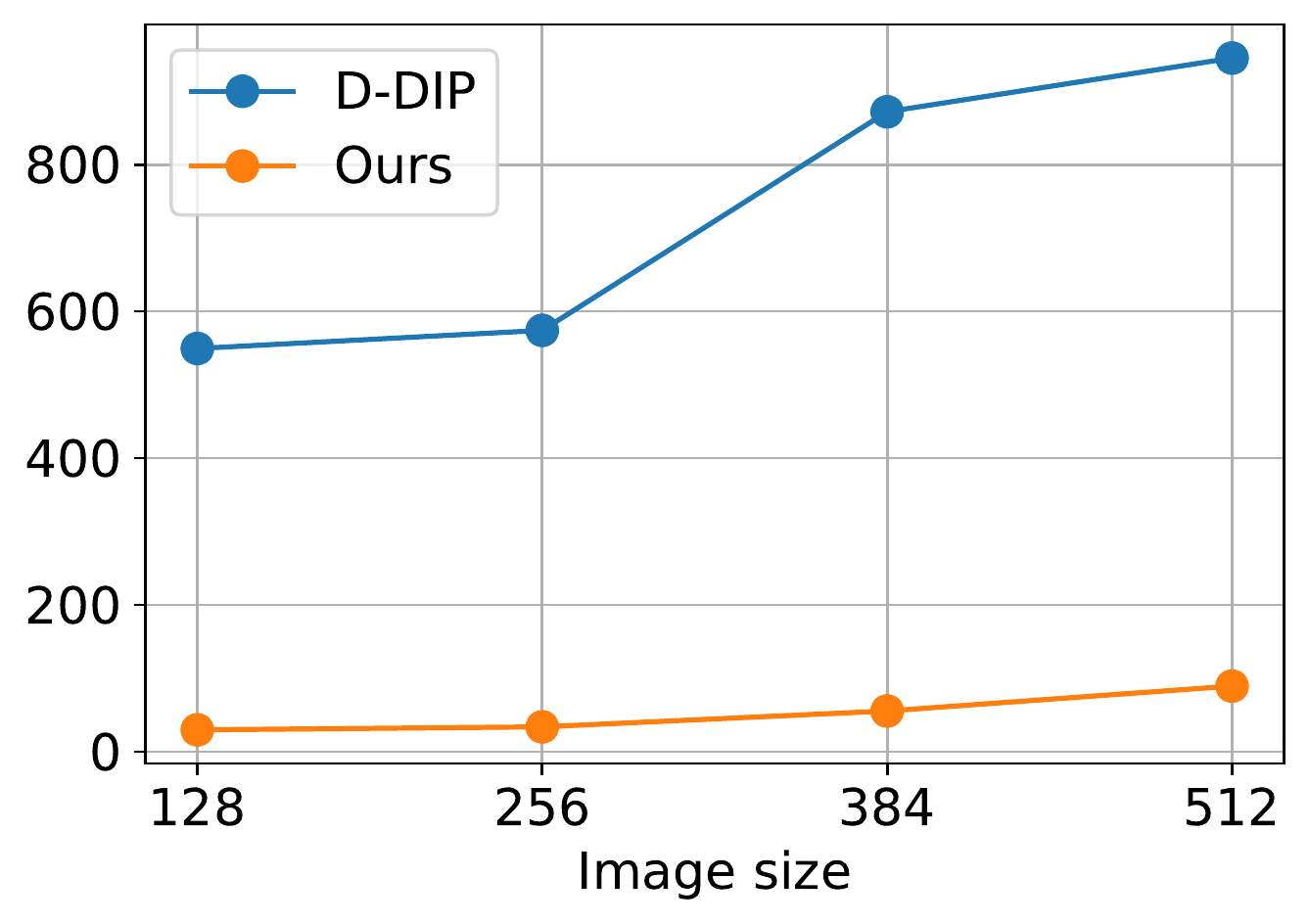} \\
        \small{(a)} & \small{(b)} \\
    \end{tabular}
    \caption{(a) segmentation results of different number of samples of $\eta$ on 20 images from Weizmann database. (b) total time on a  GTX 1080 TI.}
    \label{ns_tc}
\end{figure}

\noindent{\bf\underline{The Monte-Carlo sampling number for $\eta$.}}
We use the Monte-Carlo method to sample $\eta$ and estimate the expectation in the reconstruction term. The sampling number is set to 1 as suggested by the classical VAE method~\cite{kingma2013auto}. In the following experiment, we randomly choose 20 images from the Weizmann dataset~\cite{AlpertGBB07} and test the performance with different sampling numbers of $\eta$. The result is given in Figure~\ref{ns_tc} (a) and shows that increasing the sampling numbers of $\eta$ will not greatly increase the segmentation accuracy.

\noindent{\bf\underline{Running time.}} The relationship between the running time and image size is reported in Figure~\ref{ns_tc} (b) by using a single NVIDIA GeForce GTX 1080 Ti GPU. Comparing to the  Double-DIP~\cite{gandelsman2018double} model, we need less time for all image sizes. In particular, for $128\times128$ image, we need 29.73s while Double-DIP~\cite{gandelsman2018double} need 549.39s, which brings problems to actual use; also, the time increasing with image sizes in Double-DIP~\cite{gandelsman2018double} grows more drastically than our method.

\section{Conclusion}\label{conclusion}
This paper proposes the deep CV model, a variational inference based approach for unsupervised image segmentation by combining the traditional segmentation model with deep neural networks. Under the framework of variational inference, each term in the proposed objective function can be well explained, and the latent space assumption can be numerically verified. Experiments show that our proposed method is robust for the choice of architecture and noise corruption. Moreover, this idea can be extended to multi-phase segmentation and dataset based segmentation, experiment results show the promising performance of the proposed methods.

\section*{Acknowledgments}
Chenglong Bao was supported by the National Key R\&D Program of China (No.2021YFA1001300), National Natural Science Foundation of China (No.11901338), Tsinghua University Initiative Scientific Research Program. Zuoqiang Shi was supported by the National Natural Science Foundation of China (No.12071244).

\bibliographystyle{siam}
\bibliography{refs}

\section*{Appendix}
\subsection{Derivation of Proposition~\ref{MainProp}}\label{proof_of_main_prop}
\begin{proof} 
	We first estimate the first term in ELBO~\eqref{ELBO}. Using the the reparameterization trick~\cite{kingma2013auto}, it has
	\begin{equation}
		Z\mid I = \mathcal{G}^{\mu}(I) + \sqrt{\mathcal{G}^{\sigma}(I)}\eta,
	\end{equation}
	where $\eta \sim \mathcal{N}(0,\mathbf{I})$. Then we have
	\begin{equation}
		\begin{aligned}
			E_{q(Z\mid I)}\ln p(I\mid Z,u)
			&= E_{q(Z\mid I)}\ln p(I\mid Z) \\
			&= -\frac{1}{2}E_{q(Z\mid I)} \|\mathcal{F}(Z) - I\|^2 + c \\
			&= -\frac{1}{2}E_{\eta} \|\mathcal{F}(\mathcal{G}^{\mu}(I)+ \sqrt{\mathcal{G}^{\sigma}(I)}\eta) - I\|^2 + c, \\
		\end{aligned}
	\end{equation}
	where $c=-\ln\sqrt{2\pi}nm$. The second term is the KL divergence between two Gaussian distributions, $q(Z\mid I) = \prod_x q(Z_x\mid I)=\prod_x \mathcal{N}(\mathcal{G}^{\mu}(I)_x, \mathcal{G}^{\sigma}(I)_x))$ and $p(Z\mid u)=\prod_x p(Z_x\mid u_x)$, where $	 p(Z_x\mid u_x) = 
    \begin{cases}
    \mathcal N(\mu_1, \Sigma_1), \mbox{ if } u_x = 1 \\
    \mathcal N(\mu_2, \Sigma_2), \mbox{ if } u_x = 0
    \end{cases}$, then
	\begin{equation}
		\mathrm{KL}(q(Z\mid I)\|p(Z\mid u)) = \sum_{x\in \Omega} \mathrm{KL}(q(Z_x\mid I)\|p(Z_x\mid u_x)).
	\end{equation}
	Then $\forall x \in \Omega$,
    \begin{equation}
        \mathrm{KL}(q(Z_x\mid I)\|p(Z_x|u_x)) = \mathrm{KL}(\mathcal{N}(\mathcal{G}^{\mu}(I)_x, \mathcal{G}^{\sigma}(I)_x)) \| \mathcal{N}(\mu_i, \Sigma_i))
    \end{equation}
    which has a closed form solution~\cite{kingma2013auto}:
    \begin{equation}
        \frac{1}{2} \left(\ln \frac{|\Sigma_i|}{|\mathcal{G}^{\sigma}(I)_x|} - d + \operatorname{tr} (\Sigma_i^{-1}\mathcal{G}^{\sigma}(I)_x) + (\mathcal{G}^{\mu}(I)_x - \mu_i)^{T} \Sigma_i^{-1} (\mathcal{G}^{\mu}(I)_x - \mu_i)\right),
    \end{equation}
    where $i=1$ if $u_x=1$, $i=2$ if $u_x=0$. Sum over indexes $x \in \Omega$, the KL term is
	\begin{equation}
			\mathrm{KL}(q(Z\mid I)\|p(Z\mid u)) = \sum_{x\in \Omega} u_x\mathrm{KL}_x^{\Omega_1} + (1-u_x)\mathrm{KL}_{x}^{\Omega_2},
	\end{equation}
	where
    \begin{equation}
        \mathrm{KL}_{x}^{\Omega_i} = \frac{1}{2} \left(\ln \frac{|\Sigma_i|}{|\mathcal{G}^{\sigma}(I)_x|} - d + \operatorname{tr} (\Sigma_i^{-1}\mathcal{G}^{\sigma}(I)_x) + (\mathcal{G}^{\mu}(I)_x - \mu_i)^{T} \Sigma_i^{-1} (\mathcal{G}^{\mu}(I)_x - \mu_i)\right), i=1,2.
    \end{equation}
	 So the proposition holds.
\end{proof}

\subsection{Derivation of Theorem~\ref{Thm:Seqconver}}\label{proof_main_thm}
In addition, the encoder and decoder networks are differentiable when choosing a differentiable activation function such as sigmoid function, CReLU. 
\begin{prop}\label{pro:seqproperty}
	Suppose Assumption~\ref{assum1} holds. Let $\{x^k\}=\{(\theta^k,\gamma^k,\phi^k,w^k)\}$ be the sequence generated by Algorithm~\ref{alg:VAE_seg} and $\alpha_i\leq 2/L_M$ for $i=1,2,3$. Then, there exist $c_0,c_1>0$ such that
	\begin{align}
		&E_{\LS} (x^k)-E_{\LS}(x^{k+1})\geq c_0\|x^k-x^{k+1}\|^2, \label{energydescent}\\
		&\dist(\vzero,\partial E_{\LS}(x^{k+1}))\leq c_1\|x^k-x^{k+1}\|,\label{gradbound}
	\end{align}
	where $ \dist( \vzero, \partial E_{\LS}(x^{k+1})) :=\inf\{\|v\|:v\in\partial E_{\LS}(x^{k+1})\} $.
\end{prop}
\begin{proof} 
	Let $ y^k = (\theta^k,\gamma^k,\phi^k) $. Applying  Lemma 3.2\,\cite{chen1993convergence} to \eqref{sub:w},  we know 
	\begin{equation*}
		\nu\|w^{k+1}\|_{1,2}+\frac{\lambda}{2}\|w^{k+1}-\nabla S(\phi^{k+1})\|^2\leq  \nu\|w\|_{1,2}+\frac{\lambda}{2}\|w-\nabla S(\phi^{k+1})\|^2 - \dfrac{\lambda}{2}\|w-w^{k+1}\|^2,\quad \forall~w,
	\end{equation*}
	which implies
	\begin{equation*}
		E_{\LS}(x^{k+1})\leq E_{\LS}(y^{k+1}, w) - \dfrac{\lambda}{2}\|w-w^{k+1}\|^2 \quad \forall~w.
	\end{equation*}
	Thus, one can obtain
	\begin{align*}
		E_{\LS}(x^{k+1}) &\leq E_{\LS}(y^{k+1},w^k) - \dfrac{\lambda}{2}\|w^k - w^{k+1}\|^2\\
		&\leq  E_{\LS}(x^k) + \langle \nabla_{\theta,\gamma,\phi} E_{\LS}(x^k), y^{k+1}-y^k\rangle + \dfrac{L_M}{2}\|y^{k+1}-y^k\|^2- \dfrac{\lambda}{2}\|w^k - w^{k+1}\|^2\\
		&= E_{\LS}(x^k) -\left(\dfrac{1}{\alpha_1}- \dfrac{L_M}{2}\right)\|\theta^k- \theta^{k+1}\|^2  - \left(\dfrac{1}{\alpha_2}-\dfrac{L_M}{2} \right )\|\gamma^k-\gamma^{k+1}\|^2\\
		&\quad  - \left(\dfrac{1}{\alpha_3} - \dfrac{L_M}{2}\right)\|\phi^k-\phi^{k+1}\|^2- \dfrac{\lambda}{2}\|w^k - w^{k+1}\|^2\\
		&\leq   E_{\LS}(x^k) - c_0\|x^k - x^{k+1}\|^2,
	\end{align*} 
	where $ c_0 := \min\left\{\left(\dfrac{1}{\max_i\alpha_i} - \dfrac{L_M}{2}\right), \dfrac{\lambda}{2}\right\} $, the second inequality and  the  equality hold due to \eqref{Lipschitz} and \eqref{update}, respectively. 
	
	By the fact that $ w^{k+1} = \arg\min_w E_{\LS}(y^{k+1}, w) $, it clear that $0\in\partial_w E_{\LS}(y^{k+1}, w^{k+1}) =  \partial_w E_{\LS}(x^{k+1}) $. Then, we get
	\begin{align*}
		\dist(\vzero, \partial E_{\LS}(x^{k+1})) &=\inf\{\|v\|:v\in\partial E_{\LS}(x^{k+1})\}  \\
		&\leq \inf\{\| \nabla_{\theta,\gamma,\phi} E_{\LS}(x^{k+1})\| + \|\omega\|: \omega\in\partial_wE_{\LS}(x^{k+1})\}\\
		&\leq \| \nabla_{\theta,\gamma,\phi} E_{\LS}(x^{k+1})\|\\
		& = \|((\theta^k - \theta^{k+1})/\alpha_1, (\gamma^k - \gamma^{k+1})/\alpha_2, (\phi^k- \phi^{k+1})/\alpha_3)\|\\
		&\leq c_1 \|x^k - x^{k+1}\|,
	\end{align*}
	where $ c_1 = 1/\min_i\alpha_i $.
\end{proof}

\begin{thm}\label{Thm:subconvergence}
	Suppose Assumption~\ref{assum1} holds.  Let $\{x^k\}=\{(\theta^k,\gamma^k,\phi^k,w^k)\}$ be the sequence generated by Algorithm~\ref{alg:VAE_seg}. Then, for any limit point $x^*$ of $\{x^k\}$, we have $\vzero\in\partial E_{\LS}(x^*)$.	
\end{thm}

\begin{proof}
	By assumption, we know $\{x^k\}\subset\mathcal{M}$ and thus bounded. Then, the set of limit points of $\{x^k\}$ is nonempty. Let $ y = (\theta,\gamma, \phi) $. For any limit point $x^* = (y^*, w^*) $, there exist a subsequence $\{x^{k_j}\}$ such that $x^{k_j}\to x^*$ as $j\to\infty$. By Proposition~\ref{pro:seqproperty}, we know $\{E_{\LS}(x^k)\}$ is a decreasing sequence. Together with the fact that $E_{\LS}$ is bounded below, there exists some $\bar E$ such that $E_{\LS}(x^k)\to\bar E$ as $k\to\infty$. 
	Moverover, it has
	\begin{equation}
		E_{\LS}(x^0)-\bar E = \lim_{K\to\infty}\sum_{j=0}^{K}\left(E_{\LS}(x^j)-E_{\LS}(x^{j+1})\right)\geq c_0\lim_{K\to\infty}\sum_{j=0}^K\|x^j-x^{j+1}\|^2,
	\end{equation}
	and implies $\|x^k-x^{k-1}\|\to0$ as $k\to\infty$. 
	Together with \eqref{gradbound}, 
	it follows that there exists $v^{k_j}\in\partial_w E_{\LS}(x^{k_j})$ such that
	\begin{equation}\label{sublim}
		\lim_{j\to\infty}\|v^{k_j}\|=0.
	\end{equation}
	
	From \eqref{sub:w}, we know $ E(x^{k_j})\leq E(y^{k_j}, w) $ for any $ w $. Let $ w = w^* $ and $ j\to \infty  $, we get $ \limsup\limits_{j\to\infty} E_{\LS}(x^{k_j}) \leq E_{\LS}(x^*) $ as $E_{\LS}(y, w) $ is continuous with respect to $ y $. By the fact that $ E_{\LS}(x) $ is lower semi-continuous, it has $ \lim\limits_{j\to\infty}E_{\LS}(x^{k_j}) = E_{\LS}(x^*) $. Moreover, by the convexity of $E_{\LS}$ with respect to $ w $, we have
	\begin{equation}\label{convex}
		E_{\LS}(y^{k_j}, w) \geq E_{\LS}(x^{k_j}) + \langle v^{k_j}, w -w^{k_j}\rangle,\quad \forall~ v^{k_j}\in \partial_w E_{\LS}(x^{k_j}).
	\end{equation}
	Let $j\to\infty$ in~\eqref{convex} and using the fact that $w^{k_j}\to w^*$,  $E_{\LS}(x^{k_j})\to E_{\LS}(x^*)$ as $j\to\infty$ and~\eqref{sublim}, we get $ \vzero \in\partial_w E_{\LS}(x^*) $. Then, it follows $ \vzero\in\partial E_{\LS}(x^*) $. 
\end{proof}

Furthermore, the sub-sequence convergence can be strengthen by using the next proposition on $E_{\LS}$ which is known as the Kurdyka-Lojasiewicz (KL) property~\cite{bolte2014proximal}. In particular, we have Theorem 2.9 in~\cite{attouch2013convergence}:
\begin{thm}\label{klconvergence}
	Let $f:\mathbb{R}^n \to \mathbb{R}\cup\{+\infty\}$ be a proper lower semicontinuous function. Consider a sequence $(x^k)_{k\in \mathbb{N}}$ that satisfies, for $a,b$ are positive constants,
	
	\begin{itemize}
		\item[\noindent{\bf {H1.}}] (Sufficient decrease condition.) For each $k \in \mathbb{N}$, 
		\begin{equation}
			f(x^k) - f(x^{k+1}) \geq a\|x^{k+1} - x^{k}\|;
		\end{equation}
		\item[\noindent{\bf {H2.}}]  (Relative error condition) For each $k \in \mathbb{N}$, there exists $w^{k+1} \in \partial f(x^{k+1})$ such that
		\begin{equation}
			\|w^{k+1}\| \leq b\|x^{k+1} - x^{k}\|;
		\end{equation}
		\item[\noindent{\bf {H3.}}]  (Continuity condition). There exists a subsequence $(x^{k_j})_{j\in \mathbb{N}}$ and $\tilde{x}$ such that
		\begin{equation}
			x^{k_j} \to \tilde{x} \text{ } and \text{ } f(x^{k_j}) \to f(\tilde{x}), \quad as \text{ } j\to \infty.
		\end{equation}
	\end{itemize}
	If f has the Kurdyka-Lojasiewicz (KL) property at the cluster point $\tilde{x}$ specified in \noindent{\bf {H3}} then the sequence $(x^k)_{k \in \mathbb{N}}$converges to $\bar{x} = \tilde{x}$ as $k$ goes to infinity, and $\bar{x}$ is a critical point of $f$. Moreover the sequence $(x^k)_{k \in \mathbb{N}}$ has a finite length, i.e.,
	\begin{equation}
		\sum_{k=0}^{\infty} \|x^{k+1} - x^{k}\| < +\infty.
	\end{equation}
\end{thm}

From the Proposition 2 in~\cite{zeng2019global}, the objective function $E_{\LS}$ in \eqref{relax:loss} satisfies the KL property at $x^*$. Moreover, $E_{\LS}(x)$ and $\{x^k\}$ satisfy \noindent{\bf {H1}}, \noindent{\bf {H2}}, and \noindent{\bf {H3}} from Proposition~\ref{pro:seqproperty} and Theorem~\ref{Thm:subconvergence} and the continuity of $E_{\LS}$. Thus, from Theorem~\ref{klconvergence} and Theorem~\ref{Thm:subconvergence}, we have the sequence$\{x^k\}$ converges to $x^*$, which is a stationary point of $E_{\LS}$.

\subsection{Variational inference for dataset based segmentation}\label{app:vi_dataset}
We give the dataset based segmentation loss function from the perspective of variational inference. Assume image $I$ has two latent variables: the latent image $Z$ and the segmentation mask $u$. Then the data likelihood has
\begin{equation}\label{elbo_dataset_1}
\begin{aligned}
\ln p(I) &\geq E_{q(Z, u \mid I)} \ln\left(\frac{ p(I,Z,u)}{q(Z,u\mid I)}\right) = \underbrace{E_{q(Z, u \mid I)} \ln p(I \mid Z, u)- \mathrm{KL} (q(Z, u \mid I) \| p(Z, u))}_{\mathrm{ELBO}},
\end{aligned}
\end{equation}
and the the second term in~\eqref{elbo_dataset_1} is 
\begin{equation}\label{elbo_dataset_2}
\begin{aligned}
    \mathrm{KL} (q(Z, u \mid I) \| p(Z, u)) =& \int \int q(Z, u \mid I) \ln \frac{q(Z, u \mid I)}{p(Z, u)} dZ du \\
    =& \int \int q(u \mid I) q(Z \mid I, u)  \ln \frac{q(u \mid I) q(Z \mid I, u)}{p(Z\mid u) p(u)} dZ du \\
    =& \int q(u \mid I) \ln \frac{q(u \mid I)}{p(u)} \left( \int q(Z \mid I, u)dZ \right) du \\
    &+\int q(u \mid I) \left( \int q(Z \mid I, u) \ln \frac{q(Z \mid I, u)}{p(Z \mid u)} dZ \right) du \\
    =& \mathrm{KL} (q(u \mid I) \| p(u))  + E_{q(u \mid I)} \mathrm{KL} (q(Z \mid I, u) \| p(Z \mid u)).
\end{aligned}
\end{equation}
Combing \eqref{elbo_dataset_1} with \eqref{elbo_dataset_2}, we obtain the ELBO term as:
\begin{equation}\label{elbo_dataset}
    \mathrm{ELBO} = E_{q(Z, u \mid I)} \ln p(I \mid Z, u)- \mathrm{KL} (q(u \mid I) \| p(u)) - E_{q(u \mid I)} \mathrm{KL} (q(Z \mid I, u) \| p(Z \mid u)) .
\end{equation}
Compared to the ELBO defined in~\eqref{ELBO}, the additional term $q(u\mid I)$ is imposed for get the segmentation function $\mathcal{U}$ from variational inference methods. More specifically, we choose $q(u\mid I)$ as
\begin{equation}\label{inference_mask}
    q(u \mid I)=\delta(\mathcal{U}(I)),
\end{equation}
where $\delta$ is the Delta distribution. For other distributions in~\eqref{elbo_dataset}, recall from section~\ref{deep_cv}, we adopt the Gaussian hypothesis in the latent space:
\begin{equation}\label{latent_Gaussian}
	p(Z\mid u)=\prod_{x\in\Omega}p(Z_x\mid u_x), \quad 
	 p(Z_x\mid u_x) = 
    \begin{cases}
    \mathcal N(\mu_1, \Sigma_1), \mbox{ if } u_x = 1, \\
    \mathcal N(\mu_2, \Sigma_2), \mbox{ if } u_x = 0,
    \end{cases}
\end{equation} 
and 
\begin{equation}\label{inference_generate}
\begin{aligned}
&p(I \mid Z, u)=p(I \mid Z)=\mathcal{N}(\mathcal{F}(Z), \mathbf{I}), \\
&q(Z\mid I, u) = \prod_{x\in\Omega} q(Z_x\mid I) = \prod_{x\in\Omega} \mathcal{N}(\mathcal{G}^{\mu}(I)_x,\mathcal{G}^{\sigma}(I)_x).
\end{aligned}
\end{equation}
As before, we choose 
\begin{equation}\label{prior}
    p(u) \propto \exp (-\mathcal{R}(u)),
\end{equation}
where $\mathcal{R}$ denotes some regularizations of the segmentation mask. Now, with all these assumptions, we have
\begin{prop}\label{dataset_prop}
    Suppose the latent variable $Z$ satisfies the Gaussian hypothesis in~\eqref{latent_Gaussian}, $q(u \mid I)$ satisfies~\eqref{inference_mask}, and $p(I \mid Z, u)$, $q(Z \mid I, u)$ satisfy~\eqref{inference_generate}, then the ELBO in~\eqref{elbo_dataset} is equal to
    \begin{equation}
		-\underbrace{\frac{1}{2} E_{\eta}\|\mathcal{F}(\mathcal{G}^{\mu}(I)+ \sqrt{\mathcal{G}^{\sigma}(I)} \eta)-I\|_{2}^{2}}_{\mathrm{Reconstruction}}
		- \underbrace{\sum_{x\in \Omega} \mathcal{U}(I)_{x} \mathrm{KL}_{x}^{\Omega_1} - \left(1-\mathcal{U}(I)_{x}\right) \mathrm{KL}_{x}^{\Omega_2}}_{\mathrm{KL}} - \underbrace{\mathcal{R}(\mathcal{U}(I))}_{\mathrm{Regularization}} + c,
    \end{equation}
    where $\eta \sim \mathcal{N}(0,\mathbf{I})$, $c$ is a constant, and
    \begin{equation}
        \mathrm{KL}_{x}^{\Omega_i} = \frac{1}{2} \left(\ln \frac{|\Sigma_i|}{|\mathcal{G}^{\sigma}(I)_x|} - d + \operatorname{tr} (\Sigma_i^{-1}\mathcal{G}^{\sigma}(I)_x) + (\mathcal{G}^{\mu}(I)_x - \mu_i)^{T} \Sigma_i^{-1} (\mathcal{G}^{\mu}(I)_x - \mu_i)\right), i=1,2.
    \end{equation}
\end{prop}

\begin{proof}
    There are three terms in~\eqref{elbo_dataset}. For the first term:
    \begin{equation}
        E_{q(Z, u \mid I)} \ln p(I \mid Z, u) = E_{q(Z, u \mid I)} \ln p(I \mid Z) = E_{q(Z \mid I)} \ln p(I \mid Z),
    \end{equation}
    since we assume $p(I \mid Z, u)=p(I \mid Z)$ and $q(Z \mid I, u) = q(Z \mid I)$ in~\eqref{inference_generate}. Then we have
    \begin{equation}
        E_{q(Z \mid I)} \ln p(I \mid Z) = - \frac{1}{2} E_{\eta}\|\mathcal{F}(\mathcal{G}^{\mu}(I)+ \sqrt{\mathcal{G}^{\sigma}(I)} \eta)-I\|_{2}^{2} + c.
    \end{equation}
    where $c$ is a constant.
    For the second term:
    \begin{equation}
        \mathrm{KL} (q(u \mid I) \| p(u)) = \int q(u \mid I) \ln{\frac{q(u\mid I)}{p(u)}} d u = - H(q(u \mid I)) - E_{q(u\mid I)} \ln p(u),
    \end{equation}
    where $H(q(u \mid I))$ is the entropy of $q(u \mid I)$. Since $q(u\mid I) = \delta(\mathcal{U}(I))$ and $u\in\{0,1\}^{n\times m}$, thus $q(u\mid I)$ is a discretized distribution, so we have $H(q(u \mid I)) = H(\delta(\mathcal{U}(I)) = 0$. Futhermore, we have
    \begin{equation}
        E_{q(u\mid I)} \ln p(u) = - \mathcal{R}(\mathcal{U}(I)) - \tilde{c}
    \end{equation}
    since $p(u) \propto \exp (-\mathcal{R}(u))$, and $\tilde{c}$ is a constant. so 
    \begin{equation}
        \mathrm{KL} (q(u \mid I) \| p(u)) = \mathcal{R}(\mathcal{U}(I)) + \tilde{c}.
    \end{equation}
    For the third term, we have
    \begin{equation}
    \begin{aligned}
        E_{q(u \mid I)} \mathrm{KL} (q(Z \mid I, u) \| p(Z \mid u)) &= \mathrm{KL} (q(Z \mid I) \| p(Z \mid \mathcal{U}(I))) \\
        &= \sum_{x\in \Omega} \mathcal{U}(I)_{x} \mathrm{KL}_{x}^{\Omega_1} - \left(1-\mathcal{U}(I)_{x}\right) \mathrm{KL}_{x}^{\Omega_2}
    \end{aligned}
    \end{equation}
    where the last equation is the same as the derivation of Proposition 3.1. So the proposition holds.
\end{proof}
We can obtain the objective function for our dataset segmentation by minimizing the negative ELBO in Proposition~\ref{dataset_prop}, which is the same as the loss function we proposed in~\eqref{Energy_Functional_dataset}.

\end{document}